\IfFileExists{arxiv/arxiv.tex}{\def\versionroot{arxiv/}}{\def\versionroot{}}
\makeatletter
\edef\input@path{{\versionroot}}
\makeatother
\PassOptionsToPackage{pdftex,table}{xcolor}
\documentclass[11pt,a4paper]{gdm_format}
\usepackage[authoryear,sort&compress,round]{natbib}

\usepackage{amsmath,amsfonts,bm}

\def\eqref#1{equation~\ref{#1}}

\def\1{\bm{1}}

\DeclareMathAlphabet{\mathsfit}{\encodingdefault}{\sfdefault}{m}{sl}
\SetMathAlphabet{\mathsfit}{bold}{\encodingdefault}{\sfdefault}{bx}{n}

\renewcommand{\eqref}[1]{\textup{(\ref{#1})}}

\usepackage{tikz}
\usepackage{wrapfig}
\usepackage{algorithm}
\usepackage{algpseudocode}
\usepackage{placeins}
\usepackage{float}
\usepackage{eso-pic}
\definecolor{rscblue}{RGB}{229,241,255}
\hypersetup{
  colorlinks=true,
  citecolor={black!60},
  linkcolor={blue!50!black},
  urlcolor={blue!50!black},
  pdftitle={Neural Spectral Capacity with Residual Context for Heterogeneous Resource Allocation within Transformers},
  pdfauthor={Chenyu Zhu, Ruoyu Zhao, Zhichao Lu}
}

\newtheorem{theorem}{Theorem}

\newtheorem{proposition}{Proposition}

\titlespacing*{\paragraph}{0pt}{\parskip}{1em}

\newcommand{\TitlePageLogo}{%
  \AddToShipoutPictureBG*{%
    \AtPageUpperLeft{%
      \put(\dimexpr 0.75in\relax,-\dimexpr 0.8in\relax){%
        \includegraphics[width=5cm]{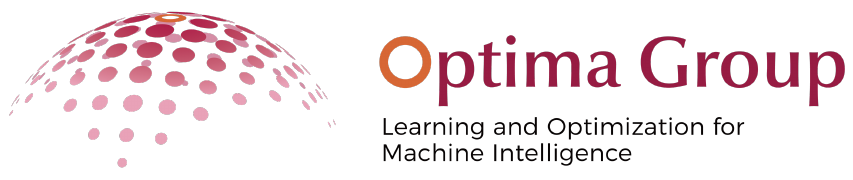}%
      }%
    }%
  }%
}

\title{Neural Spectral Capacity with\\Residual Context for Heterogeneous\\Resource Allocation within Transformers}
\correspondingauthor{Zhichao Lu (zhichao.lu@cityu.edu.hk)}
\author[ ]{Chenyu Zhu}
\author[ ]{Ruoyu Zhao}
\author[ ]{Zhichao Lu}
\affil[ ]{Department of Computer Science, City University of Hong Kong}
\renewcommand{\sectionmark}[1]{}
\renewcommand{\subsectionmark}[1]{}
\renewcommand{\bibsection}{\section*{\refname}}


\graphicspath{{\versionroot}}
\begin{document}
\raggedbottom
\TitlePageLogo
\begin{abstract}
Predicting relative trained performance from architectural specifications can guide design when training costs limit experimentation.
Neural Spectral Capacity (NSC) provides such a heuristic, but scores projections independently and assigns identical scores to compatible reorderings of the same layer configurations.
We introduce Residual-Aware Neural Spectral Capacity (ResNSC) for Transformers, using a specification-derived statistical reference for accumulated residual updates to determine the input and residual-state statistics in each projection's capacity calculation, while retaining additive evaluation without data or instantiated weights.
When residual widths vary, earlier choices affect later scores, so we augment NSC's allocator (NSC-DP) with residual history beyond its layer-and-budget state.
ResNSC-DP jointly allocates residual and feed-forward widths under a resource budget, with certified bounds on the best numerically evaluated ResNSC score within a specified finite candidate set.
At the 121M and 374M pretraining scales, selected architectures achieve lower mean perplexity than uniform, variable-width, and tapered allocations on FineWeb-Edu, WikiText-2, and LAMBADA at comparable parameter budgets.
In inherited-weight LLaMA-7B pruning using the LoNAS supernet, ResNSC-DP improves eight-task average accuracy over NSC by 3.28--5.46 percentage points across eight matched parameter budgets.
Paired reversals preserve the projection shapes and NSC scores yet worsen performance in both settings, exposing performance differences hidden by position-independent scoring.
The code for this work is available at \url{https://github.com/Optima-CityU/neural-spectral-capacity}.

\end{abstract}
\maketitle
\section{Introduction}
\label{sec:intro}

Predicting a neural network's performance before training can guide architecture development, especially when foundation-model training costs limit experimentation. Empirical scaling laws, calibrated on training runs, forecast loss from model and data scale and guide compute allocation \citep{kaplan2020scaling,hoffmann2022chinchilla}. However, these aggregate quantities leave structural differences at a given scale unresolved: equally sized networks can differ in depth, projection dimensions, and arrangement. Neural Spectral Capacity (NSC; \citealp{zhu2026nsc}) provides a heuristic for predicting relative trained performance: it analytically computes and sums projection-level scores from matrix dimensions and prescribed initialization statistics, without data or instantiated weights. Because NSC evaluates each projection independently of its position, it assigns identical scores to compatible reorderings of the same layer configurations. Recent studies show that distributing capacity differently across depth can change language-model performance at matched parameter budgets \citep{wu2026variablewidth,bayat2026tapered}. These observations motivate incorporating architectural context into specification-only scoring.

\begin{figure}[t]
\centering
\input{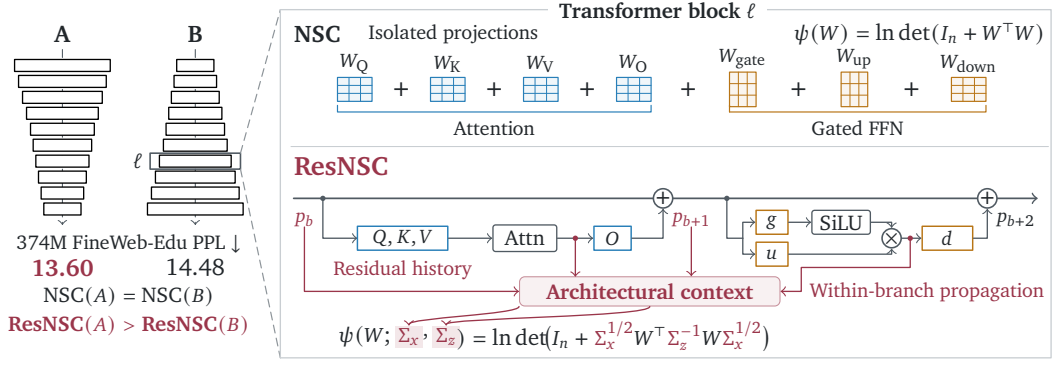}
\caption{From NSC to ResNSC. Left: schematic ResNSC-DP allocation (A) and reversal (B), with mean 374M FineWeb-Edu perplexities. Reversal preserves projection shapes and NSC scores but changes residual context. Right: additive scoring views of block $\ell$. ResNSC's specification-derived covariances use residual history and, for output projections, within-branch propagation. For $b=2\ell-2$, $p_b,p_{b+1},p_{b+2}$ are the pre-attention, post-attention, and post-FFN reference variances. Blue/orange: attention/FFN; normalization omitted.}
\label{fig:rsc-overview}
\end{figure}

The residual stream provides a concrete basis for modeling this context: in a pre-normalized Transformer, attention and feed-forward network (FFN) branches read a normalized view of the shared state and add their outputs back to it \citep{elhage2021framework,xiong2020layernorm}. Earlier branches therefore shape both the features supplied to later projections and the existing state receiving branch outputs, so identically shaped projections can operate under different conditions. In the variable-width setting considered here, branches operate on a width-dependent subset of the shared residual state, so preceding width choices determine which coordinates have been updated and how often. This raises the question: \emph{how can NSC account for this residual context while remaining computable from architecture specifications alone?}

We introduce \emph{Residual-Aware Neural Spectral Capacity} (ResNSC), which uses a specification-derived statistical reference for accumulated residual updates to determine the input and residual-state statistics in each projection's capacity calculation. For output projections, this input model includes earlier within-branch transformations (Figure~\ref{fig:rsc-overview}, right). ResNSC can therefore assign different scores to reorderings of the same projections, while retaining NSC's additive aggregation and deterministic evaluation without data or instantiated weights. When residual widths vary, ResNSC's layer scores depend on preceding choices, so the original layer-and-budget state of NSC's allocator (NSC-DP) no longer suffices. ResNSC-DP augments that state with residual history to jointly allocate residual and FFN widths under a resource budget.

At the 121M and 374M pretraining scales, architectures selected by ResNSC-DP achieve lower mean perplexity than uniform, variable-width (VWT) \citep{wu2026variablewidth}, and tapered (TLM) \citep{bayat2026tapered} allocations at comparable parameter budgets. Relative to the strongest of these baselines for each dataset and scale, the reductions are 0.3--0.4\% on FineWeb-Edu, 2.3--2.5\% on WikiText-2, and 19.6--24.9\% on LAMBADA. In LLaMA-7B pruning using the LoNAS supernet \citep{munoz2024lonas}, ResNSC-DP improves average accuracy over eight tasks by 3.28--5.46 percentage points relative to NSC across eight matched parameter budgets. In paired controls, reversing the selected width allocations preserves the collection of projection shapes and the NSC score, yet worsens performance in both settings (Figure~\ref{fig:rsc-overview}, left, illustrates the 374M case).

Our contributions are:
\begin{itemize}[leftmargin=*,itemsep=0.35em,topsep=0.35em]
\interlinepenalty=10000
\item \textbf{Residual context in capacity evaluation.} A specification-derived residual reference for evaluating Transformer projections in their architectural context.
\item \textbf{Sufficient history for joint allocation.} Under our reference, counts of preceding layers at each residual width suffice for future score evaluation, enabling ResNSC-DP with bounds on the best ResNSC score within a specified finite candidate set.
\item \textbf{Controlled evidence on layer order.} Paired reversals reveal performance differences between allocations with identical NSC scores, in both pretraining and inherited-weight pruning.
\end{itemize}

\section{Related Work}
\label{sec:related}

\paragraph{Training-free architecture evaluation.}
Training-free architecture evaluation estimates candidate quality without fully training each model. Representative methods use activation diversity or gradient statistics at initialization \citep{abdelfattah2021zerocost,mellor2021naswot}; W-PCA combines parameter information with principal-component statistics of intermediate activations \citep{wang2025wpca}. These methods evaluate instantiated models using input samples. NSC computes a deterministic score from projection dimensions and prescribed initialization statistics, evaluating each projection independently of its position \citep{zhu2026nsc}. ResNSC retains this specification-only evaluation while conditioning projection-level scores on residual context.

\paragraph{Residual context in Transformers.}
The Transformer Circuits framework describes Transformer computation through branches that read and update a shared residual representation \citep{elhage2021framework}. The behavior of this shared state depends on normalization, studied through signal and gradient propagation \citep{xiong2020layernorm}, and residual mixing, which Hyper-Connections modifies through learned connections \citep{zhu2025hyperconnections}. ResNSC uses this structural perspective to construct a specification-derived reference for evaluating projections. Its statistical assumptions are modeling choices made for architecture scoring, specified in Section~\ref{sec:ledger}.

\paragraph{Layer-wise heterogeneous Transformer architectures.}
Architecture search explores layer-specific Transformer configurations, as in AutoFormer \citep{chen2021autoformer} and FlexiBERT \citep{tuli2022flexibert}. VWT and TLM prescribe depth-dependent profiles for whole-block and FFN widths, respectively \citep{wu2026variablewidth,bayat2026tapered}, while geometry-guided allocation uses forward-pass measurements of representation change \citep{mudarisov2026geometry}. Structured pruning also makes model-dependent choices, using gradient-based importance in LLM-Pruner \citep{ma2023llmpruner} or learning pruning configurations in Sheared LLaMA \citep{xia2024sheared}. ResNSC supplies a specification-derived objective for width selection in both architecture design and inherited-weight subnetwork selection. In LoNAS \citep{munoz2024lonas}, ResNSC selects FFN widths while the elastic supernet supplies the selected architecture's weights. When residual widths vary, earlier choices affect later scores, so ResNSC-DP extends NSC's budget-constrained allocation with residual history.

\section{Residual-Aware Neural Spectral Capacity}
\label{sec:method}

ResNSC extends NSC's local capacity framework with residual history and within-branch propagation. We construct reference covariances from specifications for deterministic additive scoring, then allocate widths with history-aware ResNSC-DP under a resource budget.

\subsection{Preliminaries: local capacity and NSC}
\label{sec:capacity}

\paragraph{Local capacity.}
Consider a linear transformation with weight matrix $W\in\mathbb R^{m\times n}$. Conditional on $W$, we define the Gaussian reference channel $y=Wx+z$, where $x\sim\mathcal N(0,\Sigma_x)$ and $z\sim\mathcal N(0,\Sigma_z)$ are independent, with $\Sigma_x\succeq0$ and $\Sigma_z\succ0$. The input covariance $\Sigma_x$ describes the variances and correlations of input features, while $\Sigma_z$ describes the reference noise covariance. The Gaussian mutual-information identity gives \citep{cover2006elements,telatar1999capacity}
\begin{equation}
 \psi(W;\Sigma_x,\Sigma_z):=2I(x;y\mid W)
 =\ln\det\!\left(I_n+\Sigma_x^{1/2}W^\top\Sigma_z^{-1}W\Sigma_x^{1/2}\right).
 \label{eq:localcapacity}
\end{equation}
We use natural logarithms and prescribed reference covariances; Appendix~\ref{app:channel} gives the derivation.

\paragraph{NSC's specification-only reference.}
NSC adopts isotropic inputs and unit noise, $\Sigma_x=I_n$ and $\Sigma_z=I_m$, giving its local spectral-capacity expression $\psi(W;I_n,I_m)=\psi(W)=\ln\det(I_n+W^\top W)$ \citep{zhu2026nsc}.
For specification-only scoring, NSC evaluates this expression through a deterministic spectral formula for random-matrix ensembles specified by projection dimensions and initialization variance conventions. This yields local scores without data or instantiated candidate weights. The reference conditions do not encode projection position or residual history. ResNSC retains Equation~\eqref{eq:localcapacity} while constructing the input and noise covariances from architectural context.

\subsection{Residual reference from architectural specifications}
\label{sec:ledger}

\paragraph{Architecture and reference assumptions.}
Consider $L$ pre-normalized Transformer blocks, each containing an attention branch followed by a feed-forward network (FFN) branch. Both branches in block $\ell$ read a normalized view of the first $d_\ell$ residual coordinates and add their outputs to those coordinates; the remaining coordinates are preserved, as in VWT \citep{wu2026variablewidth}. The residual state has width $M=\max(E,d_1,\ldots,d_L)$, where $E$ is the embedding width. In our statistical reference, embedding coordinates are centered, uncorrelated, and have unit variance, with zeros elsewhere. Each branch adds a centered update with identity covariance on its active coordinates, uncorrelated with the initial state and other updates.

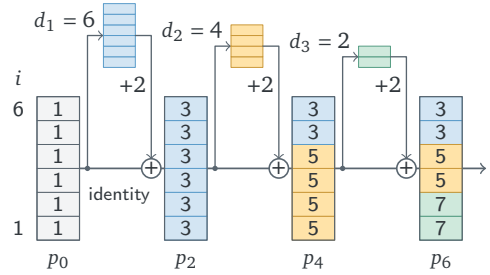
\begin{wrapfigure}{r}{0.46\textwidth}
\centering
\vspace{-1em}
\begingroup
\definecolor{resnscHistoryText}{HTML}{30363B}
\definecolor{resnscHistoryLine}{HTML}{59656F}
\definecolor{resnscHistoryInitial}{HTML}{F3F4F5}
\definecolor{resnscHistoryBlue}{HTML}{D3E5F3}
\definecolor{resnscHistoryBlueLine}{HTML}{5B9CCA}
\definecolor{resnscHistoryOrange}{HTML}{FFE3A8}
\definecolor{resnscHistoryOrangeLine}{HTML}{CB993A}
\definecolor{resnscHistoryGreen}{HTML}{CFE9DC}
\definecolor{resnscHistoryGreenLine}{HTML}{559D83}
\newcommand{\resnscLedgerCell}[5]{%
  \path[fill=#4,draw=#5,line width=0.3pt]
    (#1,{9*#2-27}) rectangle ++(16,9);
  \node[inner sep=0pt] at ({#1+8},{9*#2-22.5}) {#3};%
}
\begin{tikzpicture}[
  x=1pt,y=1pt,
  font=\sffamily\fontsize{8.5}{9.5}\selectfont,
  text=resnscHistoryText,
  line cap=round,line join=round
]
  \draw[->,resnscHistoryLine,line width=0.65pt] (16,0) -- (169,0);
  \node[font=\sffamily\fontsize{7}{8}\selectfont,inner sep=0pt]
    at (31,-10) {identity};

  \foreach \offset/\width/\block/\fillcolor/\edgecolor in {
    0/6/1/resnscHistoryBlue/resnscHistoryBlueLine,
    48/4/2/resnscHistoryOrange/resnscHistoryOrangeLine,
    96/2/3/resnscHistoryGreen/resnscHistoryGreenLine} {
    \pgfmathsetmacro{\branchmid}{38+2*\width}
    \fill[resnscHistoryLine] ({\offset+19},0) circle (1pt);
    \draw[->,resnscHistoryLine,line width=0.5pt]
      ({\offset+19},0) -- ({\offset+19},\branchmid)
      -- ({\offset+25},\branchmid);
    \foreach \row in {1,...,\width} {
      \path[fill=\fillcolor,draw=\edgecolor,line width=0.3pt]
        ({\offset+25},{38+4*(\row-1)}) rectangle ++(12,4);
    }
    \draw[->,resnscHistoryLine,line width=0.5pt]
      ({\offset+37},\branchmid) -- ({\offset+43},\branchmid)
      -- ({\offset+43},3.4);
    \node[anchor=east,inner sep=0pt] at ({\offset+22},{\branchmid+5.5}) {$d_{\block}=\width$};
    \node[inner sep=0pt] at ({\offset+36},31.5) {$+2$};
    \node[draw=resnscHistoryLine,fill=white,circle,line width=0.6pt,
      minimum size=6.8pt,inner sep=0pt,font=\fontsize{9}{9}\selectfont]
      at ({\offset+43},0) {$+$};
  }

  \foreach \row in {0,...,5} {
    \resnscLedgerCell{0}{\row}{1}{resnscHistoryInitial}{resnscHistoryLine}
    \resnscLedgerCell{48}{\row}{3}{resnscHistoryBlue}{resnscHistoryBlueLine}
  }
  \foreach \row in {0,...,3} {
    \resnscLedgerCell{96}{\row}{5}{resnscHistoryOrange}{resnscHistoryOrangeLine}
  }
  \foreach \row in {4,5} {
    \resnscLedgerCell{96}{\row}{3}{resnscHistoryBlue}{resnscHistoryBlueLine}
    \resnscLedgerCell{144}{\row}{3}{resnscHistoryBlue}{resnscHistoryBlueLine}
  }
  \foreach \row in {0,1} {
    \resnscLedgerCell{144}{\row}{7}{resnscHistoryGreen}{resnscHistoryGreenLine}
  }
  \foreach \row in {2,3} {
    \resnscLedgerCell{144}{\row}{5}{resnscHistoryOrange}{resnscHistoryOrangeLine}
  }
  \foreach \x in {0,48,96,144} {
    \draw[resnscHistoryLine,line width=0.45pt]
      (\x,-27) rectangle ++(16,54);
  }
  \foreach \x/\branch in {8/0,56/2,104/4,152/6} {
    \node[inner sep=0pt] at (\x,-35) {$p_{\branch}$};
  }
  \node[inner sep=0pt] at (-7,34) {$i$};
  \node[inner sep=0pt] at (-7,22.5) {6};
  \node[inner sep=0pt] at (-7,-22.5) {1};
\end{tikzpicture}
\endgroup
\caption{Illustrative reference variances ($E=M=6$). Each block adds $2$ (attention, then FFN) on its first $d_\ell$ coordinates; colors mark the last updating block.}
\label{fig:rsc-residual-history}
\vspace{-2em}
\end{wrapfigure}
\paragraph{Accumulated residual variance.}
Index the attention and FFN branches of block $\ell$ by $b=2\ell-2$ and $b=2\ell-1$, respectively. Denote their active widths by $\widetilde d_{2\ell-2}=\widetilde d_{2\ell-1}=d_\ell$. The residual ledger $p_b$ records the reference variance of each coordinate before branch $b$. The ledger is an auxiliary residual reference, with unit increments prescribed separately from the projection ensembles used to evaluate local capacity. Uncorrelated updates add their variances, giving
\begin{equation}
 \begin{aligned}
 p_{0,i}&=\mathbf1\{i\le E\},\quad i=1,\ldots,M,\\
 p_{b+1,i}&=p_{b,i}+\mathbf1\{i\le \widetilde d_b\}.
 \end{aligned}
 \label{eq:ledger}
\end{equation}
Each block contributes two updates, with the FFN evaluated after the attention update (Figure~\ref{fig:rsc-residual-history}). The ledger depends on preceding residual widths; FFN intermediate widths affect projection scores and resource costs but not these variance increments. Appendix~\ref{app:residual-process} gives an auxiliary process realizing this reference.

\paragraph{Reference covariances for scoring.}
Let $\bar p_b=\widetilde d_b^{-1}\sum_{i=1}^{\widetilde d_b}p_{b,i}$. Under a mean-field approximation, we replace the sample root-mean-square scale by $\sqrt{\bar p_b}$, giving normalized input variances $p_{b,i}/\bar p_b$. Read projections transform normalized residual features into branch features; write projections produce outputs added to the residual state. Normalization removes the overall variance scale while preserving relative coordinate variances. To retain sensitivity to the accumulated reference variance scale in capacity evaluation, we additionally adopt inverse-background weighting, with coordinate-wise background $1+p_{b,i}$. This evaluates each projection relative to the accumulated residual background, so the same normalized signal receives less weight when that background is larger. We divide normalized input variances by this background for read scoring and use it as the reference noise variance for write scoring. With a shared scaling constant $Z_L$, the read input and write noise covariances are
\begin{equation}
 \Sigma_{x,b}^{\mathrm{read}}=Z_L\operatorname{diag}_{i\le \widetilde d_b}\!\left(\frac{p_{b,i}}{\bar p_b(1+p_{b,i})}\right),\qquad
 \Sigma_{z,b}^{\mathrm{write}}=Z_L^{-1}\operatorname{diag}_{i\le \widetilde d_b}(1+p_{b,i}).
 \label{eq:metrics}
\end{equation}
The unit term retains NSC's unit-noise reference. When $\bar p_b=0$, we set the normalized input covariance and $\Sigma_{x,b}^{\mathrm{read}}$ to zero. The constant $Z_L=2L/\sum_{b=0}^{2L-1}(b+2)^{-1}$ is shared by all candidates of depth $L$ and normalizes the mean full-width reference precision (inverse noise covariance) to one (Appendix~\ref{app:reference}). When every branch has width $E$, normalized inputs are isotropic and both $\Sigma_{x,b}^{\mathrm{read}}$ and $(\Sigma_{z,b}^{\mathrm{write}})^{-1}$ equal $Z_L/(b+2)$ times the identity. Fixed residual width makes the reference history depth-determined. When widths vary, different preceding choices can produce different coordinate-wise contexts for the same later layer.

\subsection{Projection evaluation and architecture scoring}
\label{sec:projections}

\paragraph{Read and Write Projections.}
The query, key, value, gate, and up projections use $\Sigma_x=\Sigma_{x,b}^{\mathrm{read}}$ and unit noise $\Sigma_z=I$. For each of Q, K, and V, we evaluate the stacked heads as joint observations of their shared source, retaining the per-head entry variances; Appendix~\ref{app:projections} explains the distinction from separate-head evaluation. The attention output and FFN down projections receive features transformed within their branches. We model the upstream transformations by a prescribed random propagation matrix $W_b^{\mathrm{in}}$, independent of the output matrix being scored. The attention reference is a $d_\ell\times d_\ell$ Gaussian matrix for value features; the FFN reference combines an $f_\ell\times d_\ell$ Gaussian matrix with an independent SiLU gate, where $f_\ell$ is the FFN intermediate width. For a nonzero normalized source, these reference ensembles are scaled to give unit feature variance on average. Appendix~\ref{app:projections} specifies their distributions and scaling. Conditional on $W_b^{\mathrm{in}}$, propagating the normalized residual variances gives the input covariance
\begin{equation}
 \Sigma_{x,b}^{\mathrm{write}}
 =W_b^{\mathrm{in}}\operatorname{diag}_{i\le \widetilde d_b}\!\left(\frac{p_{b,i}}{\bar p_b}\right)
   (W_b^{\mathrm{in}})^\top,
 \label{eq:latent}
\end{equation}
with $\Sigma_{x,b}^{\mathrm{write}}=0$ when $\bar p_b=0$. We pair this input covariance with the reference noise covariance $\Sigma_{z,b}^{\mathrm{write}}$. Propagation uses the normalized variances without the additional discount applied in read scoring. Table~\ref{tab:projection-contexts} summarizes the assignments for all seven projections.

\begin{wraptable}{r}{0.46\textwidth}
 \vspace{-2em}
 \centering
 \caption{Reference contexts for Equation~\eqref{eq:localcapacity}: $d=d_\ell$, $f=f_\ell$. Each row uses the ledger immediately before branch $b$.}
 \label{tab:projection-contexts}
 \small
 \setlength{\tabcolsep}{3pt}
 \renewcommand{\arraystretch}{1.1}
 \begin{tabular}{@{}lllll@{}}
 \toprule
 \shortstack[l]{Branch\\index $b$} & \shortstack[l]{Projection\\$j$} & \shortstack[l]{Shape\\$W_{\ell,j}$} & \shortstack[l]{Input\\$\Sigma_x$} & \shortstack[l]{Noise\\$\Sigma_z$}\\
 \midrule
 $2\ell-2$ & Q, K, V & $d\times d$ & $\Sigma_{x,b}^{\mathrm{read}}$ & $I_d$\\
 $2\ell-2$ & O & $d\times d$ & $\Sigma_{x,b}^{\mathrm{write}}$ & $\Sigma_{z,b}^{\mathrm{write}}$\\
 $2\ell-1$ & gate, up & $f\times d$ & $\Sigma_{x,b}^{\mathrm{read}}$ & $I_f$\\
 $2\ell-1$ & down & $d\times f$ & $\Sigma_{x,b}^{\mathrm{write}}$ & $\Sigma_{z,b}^{\mathrm{write}}$\\
 \bottomrule
 \end{tabular}
\end{wraptable}
\paragraph{Deterministic local evaluation.}
For each projection, the reference capacity is the expectation of the local functional in Equation~\eqref{eq:localcapacity} over its prescribed matrix ensemble. For write terms, this expectation includes both the output matrix and the upstream reference, retaining the propagated covariance inside the log-determinant. Read and write expectations are evaluated using covariance-ensemble and product-ensemble spectral formulas, respectively, which reduce the calculation to deterministic scalar fixed-point equations. These formulas approximate the finite-dimensional ensemble expectations from architectural dimensions, residual history, and the prescribed reference conventions, without data or instantiated weights (Appendix~\ref{app:spectral}; \citealp{hoydis2011iterative}). Let $\widehat\psi_{\ell,j}(p_b)$ denote the numerical score for projection $j$ in block $\ell$ with context $p_b$.

\paragraph{Architecture-level aggregation.}
For an architecture $a=((d_1,f_1),\ldots,(d_L,f_L))$ and projection set $\mathcal M=\{\mathrm Q,\mathrm K,\mathrm V,\mathrm O,\mathrm{gate},\mathrm{up},\mathrm{down}\}$, we define the ResNSC score by summing these numerical local scores:
\begin{equation}
 S(a)
 =\sum_{\ell=1}^{L}\sum_{j\in\mathcal M}
       \widehat\psi_{\ell,j}\!\left(p_{b(\ell,j)}\right)
 =\sum_{\ell=1}^{L}S_L^{\mathrm{block}}\!\left(p_{2\ell-2};d_\ell,f_\ell\right).
 \label{eq:architecturecapacity}
\end{equation}
Here $b(\ell,j)$ is given in Table~\ref{tab:projection-contexts}. The block score $S_L^{\mathrm{block}}(p;d,f)$ sums the seven projection terms for incoming ledger $p$, residual width $d$, and FFN width $f$, updating the ledger after attention before evaluating the FFN. Fixed embedding width, head count, and reference conventions are suppressed in this notation. Embedding support initializes the ledger; embeddings, normalization, and readout count toward the resource budget but not the seven projection terms. This additive architecture score can distinguish block reorderings because each term uses its own intermediate residual context, even when the final ledger agrees. When residual widths vary, preceding width choices affect subsequent block scores; the next subsection develops an allocator that retains the required history.

\subsection{Budget-constrained allocation with residual history}
\label{sec:search}
\label{sec:sufficiency}

\paragraph{Formulation.}
We jointly select residual and FFN widths to maximize the numerical ResNSC score under a resource budget. Let $\mathcal G$ be a finite set of specifications. Each $g$ fixes the depth $L_g$, embedding width $E_g$, head count, reference conventions, a finite menu $\mathcal A_g$ of admissible $(d,f)$ pairs, and budget $B_g$. We solve
\begin{equation}
 \max_{\substack{g\in\mathcal G\,,\ a\in\mathcal A_g^{L_g}}}S_g(a)
 \quad\text{subject to}\quad
 c_{\mathrm{fixed}}(g)+\sum_{\ell=1}^{L_g}c_{g,\ell}(d_\ell,f_\ell)\le B_g,
 \label{eq:searchproblem}
\end{equation}
where $S_g$ is the score from Section~\ref{sec:projections} under specification $g$, and the resource cost comprises a fixed contribution and nonnegative additive block costs. Our experiments use integer parameter counts, including embedding/readout interfaces and normalization.

\paragraph{Representing residual history.}
For a fixed $g$ with variable residual widths, two partial architectures (prefixes) at the same depth and spent cost can have different residual ledgers and therefore different scores for the same continuation. The layer-and-budget state used by NSC-DP is therefore insufficient for this history-dependent objective. ResNSC-DP additionally records the counts $\mathbf n=(n_1,\ldots,n_K)$ of preceding blocks at each residual width $d^{(1)}<\cdots<d^{(K)}$. Under the two-update reference, these counts reconstruct the ledger:
\begin{equation}
 p_{g,i}(\mathbf n)=\mathbf1\{i\le E_g\}
       +2\sum_{r=1}^{K}n_r\mathbf1\{i\le d^{(r)}\},
 \qquad t=\sum_{r=1}^{K}n_r.
 \label{eq:histogram}
\end{equation}
The counts determine both the current ledger and the next block index. Appending the same choice gives the same block score and next count vector, so equal counts imply identical additional scores for every common continuation. Under the stated menus and additive resource constraint, feasible continuations also agree at the same remaining budget. We store each prefix as $\pi=(g,\mathbf n,c,s)$, where $c$ is spent variable cost and $s$ is accumulated score. Within the same $(g,\mathbf n)$, a prefix with no greater cost and no smaller score dominates another. The counts retain the history needed for future evaluation, while $s$ retains the effects of past order. FFN choices enter through $c$ and $s$ (Appendix~\ref{app:history-proof}).

\paragraph{Search with completion bounds.}
ResNSC-DP uses backward dynamic programming to construct completion bounds: upper bounds on the total score of any feasible architecture extending a prefix. These bounds guide a search over nondominated prefixes. For each nonnegative resource price $\lambda$, the backward DP maximizes additional score minus $\lambda$ times additional cost, retaining history transitions while relaxing the hard budget. Adding the earned score $s$ and $\lambda$ times the remaining budget $B_g-c_{\mathrm{fixed}}(g)-c$ gives an upper bound on the score of any feasible completion. The tightest bound across the prescribed prices defines $U(\pi)$. Initialization tests minimum-cost architectures to find an initial feasible solution or establish infeasibility. Algorithm~\ref{alg:nsc-cert} expands the prefix with largest $U$ and discards infeasible or dominated prefixes and those whose bounds cannot improve the best feasible complete score found so far. Appendix~\ref{app:search} gives the bound recurrence and detailed procedure (Algorithm~\ref{alg:nsc-cert-full}).

\paragraph{Optimality guarantees.}
Let $S^*$ be the optimum of the numerical objective in Equation~\eqref{eq:searchproblem}. For a nonempty feasible set, under exact arithmetic for the fixed numerical objective and with admissible completion bounds, the best feasible architecture found so far supplies $\mathrm{LB}=S_{\widehat g}(\widehat a)$. At initialization and after each complete prefix expansion, the maximum of $\mathrm{LB}$ and the completion bounds of all remaining valid prefixes supplies $\mathrm{UB}$, with
\begin{equation}
 \mathrm{LB}\le S^*\le\mathrm{UB},\qquad
 \mathrm{UB}-\mathrm{LB}\le\epsilon
 \ \Longrightarrow\ S^*-\mathrm{LB}\le\epsilon.
 \label{eq:guarantee}
\end{equation}
A gap of at most $\epsilon$ establishes $\epsilon$-optimality, and exhausting the valid queue establishes optimality. If search stops earlier, the bounds quantify the remaining objective gap. Appendix~\ref{app:search} gives the proof and implementation details. With fixed residual width, depth alone determines the ledger. Precomputing block scores by layer and FFN width recovers ordinary layer-and-budget dynamic programming with ResNSC's contextual score (Appendix~\ref{app:fixed-width}).

\begin{algorithm}[t]
\definecolor{resnscAlgoContext}{HTML}{9A354A}
\newcommand{\resnscAlgHighlight}[1]{\textcolor{resnscAlgoContext}{#1}}
\caption{ResNSC-DP: allocation with residual history.}
\label{alg:nsc-cert}
\small
\algrenewcommand\algorithmicrequire{\textbf{Input:}}
\algrenewcommand\algorithmicensure{\textbf{Output:}}
\begin{algorithmic}[1]
\Require Specifications $\mathcal G$, menus $\mathcal A_g$, budgets $B_g$, tolerance $\epsilon\ge0$
\Ensure Feasible architecture and objective bounds, or infeasibility
\Statex $\mathcal Q$: valid live prefixes; $\mathbf e_{r(d)}$: width-$d$ count increment; \resnscAlgHighlight{burgundy}: residual-history operations.
\State Construct completion bounds $U$ by \resnscAlgHighlight{history-indexed backward DP} (Appendix~\ref{app:resource-bound}).
\State Initialize feasible roots $\mathcal Q$ and an incumbent with score $\mathrm{LB}$; return infeasible if none exists.
\While{$\exists\pi\in\mathcal Q:U(\pi)>\mathrm{LB}+\epsilon$ and search limits allow}
 \State Remove $\pi=(g,\resnscAlgHighlight{\mathbf n},c,s)$ with largest $U(\pi)$ from $\mathcal Q$.
 \State Reconstruct \resnscAlgHighlight{$p\gets p_g(\mathbf n)$} and $t\gets\sum_r n_r$.
 \For{each $(d,f)\in\mathcal A_g$}
  \State $s'\gets s+\resnscAlgHighlight{S_{L_g}^{\mathrm{block}}(p;d,f)}$.
  \State \resnscAlgHighlight{$\mathbf n'\gets\mathbf n+\mathbf e_{r(d)}$}, $c'\gets c+c_{g,t+1}(d,f)$; $\pi'\gets(g,\mathbf n',c',s')$.
  \State \textbf{if} no affordable completion, dominated \resnscAlgHighlight{within $(g,\mathbf n')$}, or $U(\pi')\le\mathrm{LB}$ \textbf{then continue}.
  \State \textbf{if} $t+1=L_g$ \textbf{then} update the incumbent and $\mathrm{LB}$ only on improvement.
  \State \textbf{else} insert $\pi'$ into $\mathcal Q$ and update the cost--score frontier at \resnscAlgHighlight{$(g,\mathbf n')$}.
 \EndFor
\EndWhile
\State \Return incumbent architecture and $(\mathrm{LB},\mathrm{UB})$, where $\mathrm{UB}=\max\bigl(\{\mathrm{LB}\}\cup\{U(\pi):\pi\in\mathcal Q\}\bigr)$.
\end{algorithmic}
\end{algorithm}

\FloatBarrier
\section{Experiments}
\label{sec:experiments}

We evaluate whether residual context captures consequential architectural differences, improves resource allocation, and retains NSC's global ranking utility. Experiments cover from-scratch pretraining at 121M and 374M, inherited-weight pruning of LoNAS-LLaMA-7B, and ranking across language and vision search spaces.

\subsection{Architecture search and pretraining}
\label{sec:pretraining}

\paragraph{Setup.}
We study 121M and 374M decoder-only Transformers with 16 and 24 blocks, jointly varying each block's residual width $d_\ell$ and FFN width $f_\ell$ without a prescribed taper. ResNSC-DP selects architectures from specifications alone under the Uniform model's total-parameter ceiling; NSC-DP \citep{zhu2026nsc} uses the same menus and ceiling. We also compare constant-width Uniform, whole-block bottleneck VWT \citep{wu2026variablewidth}, and FFN-tapered TLM \citep{bayat2026tapered} profiles at comparable parameter counts. Reverse reverses the ResNSC-DP sequence of $(d_\ell,f_\ell)$ pairs, preserving projection shapes and the NSC score. Model specifications, menus, and configurations appear in Appendix~\ref{app:widths}, and search certificates in Appendix~\ref{app:search-measurements}.

All architectures within a scale are trained from scratch on the same FineWeb-Edu subset \citep{penedo2024fineweb}, with a shared $\mu$P/AdamW recipe and token budgets of 2.662B at 121M and 8.228B at 374M. We evaluate held-out FineWeb-Edu loss and perplexity, WikiText-2 word perplexity, LAMBADA perplexity, and eight-task zero-shot accuracy at 374M. Means and sample standard deviations use three independent runs for the 121M main comparisons. The 374M experiments and 121M ablations use one full run plus two final-decay continuations, with standard deviations measuring annealing variation (Appendix~\ref{app:training-statistics}).

\begin{table}[!htb]
\providecommand{\sd}[1]{{\scriptsize($\pm$#1)}}
\centering
\caption{Language-modeling results after final decay at 121M and 374M: FineWeb-Edu perplexity (PPL), WikiText-2 word perplexity, LAMBADA perplexity, and, at 374M, mean eight-task zero-shot accuracy (Avg8). 121M: means and sample standard deviations over three independent runs; 374M: one run plus two final-decay continuations, so standard deviations reflect annealing variation only. Bold: best per column; per-task accuracies in Appendix~\ref{app:training-statistics}.}
\label{tab:rsc-pretraining}\label{tab:rsc-121M-anneal}\label{tab:rsc-340M-anneal}
\small
\setlength{\tabcolsep}{3pt}
\renewcommand{\arraystretch}{1.1}
\resizebox{\linewidth}{!}{%
\begin{tabular}{@{}l rrr rrrr@{}}
\toprule
 & \multicolumn{3}{c}{121M} & \multicolumn{4}{c}{374M} \\
\cmidrule(lr){2-4}\cmidrule(l){5-8}
Method & PPL$\downarrow$ & WikiText-2$\downarrow$ & LAMBADA$\downarrow$ & PPL$\downarrow$ & WikiText-2$\downarrow$ & LAMBADA$\downarrow$ & Avg8$\uparrow$ \\
\midrule
\rowcolor{rscblue} \textbf{ResNSC (ours)} & \textbf{19.39}\,\sd{0.01} & \textbf{51.13}\,\sd{0.21} & \textbf{202.8}\,\sd{13.5} & \textbf{13.60}\,\sd{0.002} & \textbf{31.20}\,\sd{0.03} & \textbf{53.0}\,\sd{0.4} & \textbf{44.60}\,\sd{0.07} \\
uniform & 20.25\,\sd{0.02} & 55.47\,\sd{0.07} & 262.9\,\sd{13.7} & 14.27\,\sd{0.002} & 33.82\,\sd{0.05} & 76.2\,\sd{2.0} & 44.07\,\sd{0.19} \\
VWT & 19.48\,\sd{0.06} & 52.33\,\sd{0.15} & 252.6\,\sd{35.7} & 13.65\,\sd{0.002} & 31.99\,\sd{0.08} & 70.6\,\sd{2.1} & 44.26\,\sd{0.11} \\
TLM & 20.38\,\sd{0.02} & 55.43\,\sd{0.53} & 252.1\,\sd{8.5} & 14.43\,\sd{0.001} & 34.19\,\sd{0.02} & 78.1\,\sd{0.6} & 43.73\,\sd{0.05} \\
NSC & 20.00\,\sd{0.03} & 53.40\,\sd{0.31} & 221.4\,\sd{10.1} & 13.82\,\sd{0.003} & 32.40\,\sd{0.11} & 66.9\,\sd{1.9} & 44.26\,\sd{0.09} \\
Reverse & 20.27\,\sd{0.05} & 53.22\,\sd{0.55} & 234.6\,\sd{9.3} & 14.48\,\sd{0.001} & 33.93\,\sd{0.01} & 77.7\,\sd{1.6} & 42.86\,\sd{0.18} \\
\bottomrule
\end{tabular}}
\end{table}

\begin{table}[!htb]
\providecommand{\sd}[1]{{\scriptsize($\pm$#1)}}
\centering
\caption{Ablating the architecture score at 121M. All variants share the search space and parameter budget; profiles show the widest and final active widths. Rows 2--4 simplify residual history (none, branch position only, coordinate average); the last row removes the propagated write-source spectrum. Entries are one run (seed 42) plus two final-decay continuations; standard deviations reflect annealing variation only. Bold: best per column.}
\label{tab:rsc-component-ablation}
\small
\setlength{\tabcolsep}{3pt}
\renewcommand{\arraystretch}{1.08}
\begin{tabularx}{\linewidth}{@{}l X r r r@{}}
\toprule
Scoring rule & Profile $\max_\ell d_\ell\to d_{16}$ & PPL$\downarrow$ & WikiText-2$\downarrow$ & LAMBADA$\downarrow$ \\
\midrule
\rowcolor{rscblue} \textbf{ResNSC (full)} & $1280\to480$ & \textbf{19.39}\,\sd{0.01} & \textbf{51.04}\,\sd{0.20} & 211.2\,\sd{2.3} \\
w/o residual context & $800\to640$ (identical to NSC) & 20.00\,\sd{0.01} & 53.25\,\sd{0.21} & 233.6\,\sd{1.7} \\
Depth only & $1760\to320$ & 19.85\,\sd{0.01} & 52.85\,\sd{0.21} & 205.9\,\sd{2.5} \\
Mean history only & $1920\to320$ & 20.20\,\sd{0.01} & 54.61\,\sd{0.20} & 229.5\,\sd{2.7} \\
w/o propagated spectrum & $1440\to320$ & 19.72\,\sd{0.01} & 52.72\,\sd{0.20} & \textbf{197.5}\,\sd{3.3} \\
\bottomrule
\end{tabularx}
\end{table}

\paragraph{Results.}
Reversing the ResNSC-DP allocation increases mean FineWeb-Edu perplexity from 19.39 to 20.27 at 121M and from 13.60 to 14.48 at 374M, despite unchanged NSC scores (Table~\ref{tab:rsc-pretraining}). WikiText-2 and LAMBADA perplexities also increase at both scales. These differences show that the placement of the same projection shapes affects trained performance, exposing a limitation of NSC's order-insensitive aggregation.

Architectures selected by ResNSC-DP achieve the lowest mean perplexities across all three datasets and both scales in Table~\ref{tab:rsc-pretraining}. Against NSC-DP, FineWeb-Edu perplexity decreases from 20.00 to 19.39 at 121M and from 13.82 to 13.60 at 374M. At 374M, the selected architecture also achieves the highest mean eight-task accuracy, 44.60\%, versus 44.26\% for the strongest comparators, NSC and VWT (Appendix~\ref{app:training-statistics}).

At 121M, we test removing residual context, retaining only branch position or mean reference variance, and removing the propagated spectrum, with each variant selecting an architecture under the same menus and parameter ceiling (Table~\ref{tab:rsc-component-ablation}; definitions in Appendix~\ref{app:scoring-ablations}). Full ResNSC achieves the lowest mean FineWeb-Edu and WikiText-2 perplexity, supporting coordinate-wise residual history and within-branch propagation for these evaluations. The depth-only and no-propagation variants achieve lower LAMBADA perplexity.

\subsection{Structured pruning of LoNAS-LLaMA-7B}
\label{sec:lonas}

\paragraph{Setup.}
Following NSC \citep{zhu2026nsc}, we select subnetworks from the LoNAS LLaMA-7B supernet \citep{munoz2024lonas} with 32 blocks, fixed residual width 4,096, and variable FFN widths and Q/K/V LoRA ranks. Subnetworks inherit weights from the same supernet using the same channel-slicing procedure. We report Avg8, the mean accuracy over eight commonsense tasks; Appendix~\ref{app:lonas} gives the menus, tasks, and LoRA scoring extension.

ResNSC selects from architecture specifications without candidate weights or calibration inputs. The fixed-width specialization of ResNSC-DP uses the same exact layer-and-budget recurrence as NSC-DP, with contextual layer scores. Both methods use identical menus under eight shared parameter ceilings; their selected subnetworks have matching parameter counts at each ceiling. We include the SNIP \citep{lee2019snip}, SynFlow \citep{tanaka2020synflow}, GradNorm \citep{abdelfattah2021zerocost}, and W-PCA \citep{wang2025wpca} evolutionary-search results reported by NSC under a common 5.7B parameter ceiling. Reverse assigns the highest-budget ResNSC FFN widths in reverse order across the original supernet layers, preserving parameter count, FLOPs, and NSC score.

\Needspace{20\baselineskip}
\begingroup
\setlength{\intextsep}{3pt}
\setlength{\abovecaptionskip}{3pt}
\setlength{\belowcaptionskip}{0pt}
\setlength{\columnsep}{12pt}
\setlength{\parskip}{3pt}
\begin{wrapfigure}{r}{0.48\textwidth}
\centering
\setlength{\parskip}{0pt}
\includegraphics[width=\linewidth]{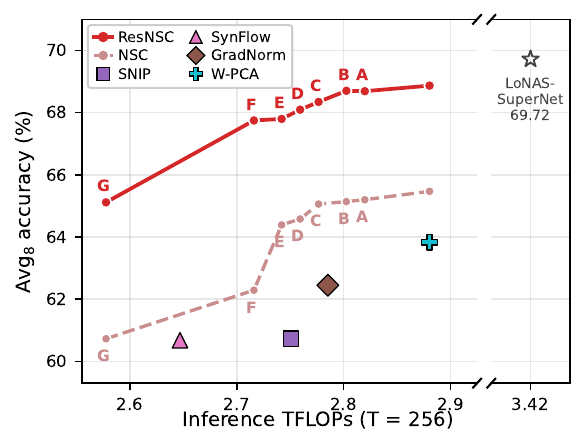}
\caption{ResNSC improves the accuracy--compute trade-off. Avg8 versus TFLOPs ($T=256$) across eight matched budgets (A--G: smaller budgets), with proxy baselines and the unpruned supernet (star). Full results in Appendix~\ref{app:lonas}.}
\label{fig:lonas-pareto}
\end{wrapfigure}
\paragraph{Results.}
Reverse achieves 59.49\% Avg8 versus 68.87\% for ResNSC, a 9.38-point difference at the same parameter count and NSC score (Table~\ref{tab:lonas-main}). The placement of retained FFN widths therefore also affects performance under weight inheritance.

At 5.690B parameters, ResNSC exceeds NSC by 3.40 points and the strongest proxy baseline, W-PCA, by 5.04 points in Avg8 (Table~\ref{tab:lonas-main}). This allocation retains 98.8\% of the unpruned supernet's Avg8 while using approximately 16\% fewer parameters. Across all eight matched NSC budgets, the gain is 3.28--5.46 points (Figure~\ref{fig:lonas-pareto}).

In this fixed-width setting, layer scoring and selection across all eight budgets take under one second, excluding supernet training and subnetwork evaluation. The NSC study reports subsecond search for NSC and 22--45 minutes for proxy baselines (Table~\ref{tab:lonas-main}).
\par\endgroup

\begin{table}[!htb]
\setlength{\abovecaptionskip}{3pt}
\setlength{\belowcaptionskip}{0pt}
\centering
\caption{LoNAS LLaMA-7B pruning on eight commonsense tasks (accuracy, \%). ResNSC and NSC use the same width/rank menu and matched parameter counts. Params reports each selected model's size. The supernet is a reference; \textbf{bold} marks the best pruned result per column. Reverse reassigns FFN widths in reverse order across the original layers, preserving parameter count and NSC score. Cost is search-stage wall-clock time.}
\label{tab:lonas-main}
\setlength{\tabcolsep}{3pt}
\resizebox{\linewidth}{!}{%
\begin{tabular}{@{}l r *{9}{r} r@{}}
\toprule
Method & Params & BoolQ & PIQA & SIQA & HSwag & WGrnd & ARC-e & ARC-c & OBQA & Avg$_8$ & Cost \\
\midrule
LoNAS-SuperNet & 6.74\,B & 66.02 & 77.58 & 72.93 & 57.49 & 66.85 & 78.41 & 62.29 & 76.20 & 69.72 & --- \\
\midrule
SNIP & 5.44\,B & 61.93 & 69.21 & 64.99 & 41.26 & 62.04 & 66.50 & 52.47 & 67.40 & 60.73 & 22 min \\
SynFlow & 5.23\,B & 61.13 & 70.24 & 63.25 & 45.94 & 59.98 & 66.84 & 51.88 & 66.20 & 60.68 & 23 min \\
GradNorm & 5.50\,B & 61.31 & 70.78 & 67.35 & 46.33 & 62.59 & 69.02 & 55.03 & 67.20 & 62.45 & 22 min \\
W-PCA & 5.69\,B & 63.82 & 70.95 & 67.91 & 49.24 & 62.75 & 69.57 & 56.40 & 70.00 & 63.83 & 45 min \\
NSC & 5.69\,B & \textbf{64.01} & 72.85 & 68.94 & 51.30 & 64.72 & 72.43 & 58.70 & 70.80 & 65.47 & $<$1 s \\
\midrule
\rowcolor{rscblue} ResNSC & 5.69\,B & 61.71 & \textbf{76.71} & \textbf{70.98} & \textbf{58.39} & \textbf{67.01} & \textbf{78.45} & \textbf{61.69} & \textbf{76.00} & \textbf{68.87} & $<$1 s \\
Reverse & 5.69\,B & 61.25 & 67.46 & 63.10 & 42.64 & 61.17 & 65.24 & 51.45 & 63.60 & 59.49 & --- \\
\bottomrule
\end{tabular}}
\end{table}

\subsection{Ranking across architecture search spaces}
\label{sec:ranking}

\paragraph{Setup.}
We examine whether incorporating residual context retains NSC's global ranking utility while distinguishing more candidate architectures. Following the NSC protocol \citep{zhu2026nsc}, we compare the scores on 500 FlexiBERT architectures \citep{serianni2023nasbench}, 200 GPT-2 architectures \citep{radford2019gpt2}, and 1,001 AutoFormer-Small candidates \citep{chen2021autoformer}, using identical candidates, trained-performance labels, and specifications. Table~\ref{tab:rsc-ranking} reports rank correlations and distinct-score counts; Appendix~\ref{app:ranking} gives the protocol and supplementary results.

\paragraph{Results.}
ResNSC retains strong global ranking correlations: Spearman $\rho$ is comparable to NSC on FlexiBERT and GPT-2 and higher on AutoFormer-Small (Table~\ref{tab:rsc-ranking}); it is lower on AutoFormer-Tiny (Appendix~\ref{app:ranking}). Distinct-score counts increase from 297 to 421 on FlexiBERT and from 952 to 1,001 on AutoFormer-Small; both assign distinct scores to all 200 GPT-2 candidates. Together with the paired reversals in Sections~\ref{sec:pretraining}--\ref{sec:lonas}, these distinctions suggest that architectural context can help guide resource allocation.

\begin{table}[ht]
\centering
\caption{Global ranking correlations and number of distinct scores.}
\label{tab:rsc-ranking}
\small
\setlength{\tabcolsep}{3pt}
\renewcommand{\arraystretch}{1.0}
\begin{tabular*}{\linewidth}{@{\extracolsep{\fill}}l r c c c@{}}
\toprule
Benchmark & $N$ & $\tau_b$ (NSC\,/\,ResNSC) & $\rho$ (NSC\,/\,ResNSC) & Distinct (NSC\,/\,ResNSC) \\
\midrule
FlexiBERT & 500 & 0.695\,/\,0.692 & 0.884\,/\,0.883 & 297\,/\,421 \\
GPT-2 & 200 & 0.849\,/\,0.854 & 0.968\,/\,0.970 & 200\,/\,200 \\
AutoFormer-Small & 1,001 & 0.658\,/\,0.719 & 0.852\,/\,0.907 & 952\,/\,1,001 \\
\bottomrule
\end{tabular*}
\end{table}

\FloatBarrier
\section{Conclusion}
\label{sec:conclusion}

We introduced Residual-Aware Neural Spectral Capacity (ResNSC), which scores each Transformer projection in the residual context implied by its architecture, using a reference constructed from the architecture specification alone. Because earlier width choices change later scores, ResNSC-DP extends NSC-DP with a residual-history state and bounds the best score within a specified finite candidate set.

Paired reversals show that allocations with identical projection shapes and NSC scores perform differently after from-scratch training and under weight inheritance. Optimizing ResNSC yields lower perplexity than uniform, VWT, TLM, and same-menu NSC allocations at 121M and 374M, and higher eight-task average accuracy than NSC at every matched LoNAS budget. Across broader search spaces, ResNSC retains strong global ranking correlations while distinguishing more candidate architectures. ResNSC keeps NSC's additive, data-free evaluation, providing a basis for studying other architecture families and larger candidate sets. Beyond the residual stream, conditioning each projection on further elements of the computational graph, such as attention mixing across positions and expert routing, is the next step toward architecture-faithful capacity evaluation from the specification alone.

\FloatBarrier
{\small
\bibliographystyle{plainnat}
\bibliography{\versionroot iclr2027_conference}
}
\appendix
\numberwithin{equation}{section}

\section{Notation and the local capacity identity}
\label{app:channel}

The main text uses $\psi$ for local capacity, $\widehat\psi$ for its deterministic spectral evaluation, $S_L^{\mathrm{block}}$ for a block sum, and $S$ for the architecture sum. Superscripts $\mathrm{read}$ and $\mathrm{write}$ distinguish the two reference-ensemble evaluations of $\psi$. Table~\ref{tab:notation} fixes the notation used throughout these appendices.

\begin{table}[htbp]
\centering
\caption{Notation shared by the main text and mathematical appendices.}
\label{tab:notation}
\small
\renewcommand{\arraystretch}{1.13}
\begin{tabularx}{\textwidth}{@{}lX@{}}
\toprule
Symbol & Meaning\\
\midrule
$\ell$, $b$ & Block index and branch index; attention has $b=2\ell-2$, FFN has $b=2\ell-1$.\\
$E,L,H,M$ & Embedding width, number of blocks, number of attention heads, and complete residual-state width.\\
$d_\ell,f_\ell$ & Active residual width and FFN intermediate width of block $\ell$.\\
$\widetilde d_b,k_b$ & Active residual width (shared by reading and writing) and latent width of branch $b$.\\
$h_b,p_b$ & Actual residual state and specification-derived reference-power ledger before branch $b$.\\
$\Sigma_x,\Sigma_z$ & Generic source and noise covariances in $\psi(W;\Sigma_x,\Sigma_z)$.\\
$\Sigma_{x,b}^{\mathrm{read}},\Sigma_{x,b}^{\mathrm{write}}$ & Read source covariance and propagated write source covariance.\\
$\Sigma_{z,b}^{\mathrm{write}}$ & Reference noise covariance for write scoring.\\
$W_b^{\mathrm{in}}$ & Normalized upstream reference, including the gate when present.\\
$\Sigma_{x,b}^{\mathrm{in}}$ & Normalized branch-input covariance (Appendices B--D); it has no read discount.\\
$X_b,G_b$ & Gaussian factor and diagonal gate used to construct $W_b^{\mathrm{in}}=G_bX_b$ in the appendices.\\
$W_{\ell,j}$ & Scored projection of type $j$ in block $\ell$.\\
$Z_L$ & Shared normalization of the full-width reference precision at depth $L$.\\
$\psi,\widehat\psi,S_L^{\mathrm{block}},S$ & Local functional, deterministic local evaluation, block sum, and architecture sum.\\
$\mathbf n,n_r,t$ & Residual-width count vector, count of width option $r$, and completed block count.\\
$g,\mathcal A_g,B_g$ & Fixed network specification, its block-choice menu, and resource budget (parameter ceiling in our experiments).\\
$c,s$ & Spent variable resource cost and accumulated prefix score.\\
$\lambda,V_{g,\lambda},U$ & Resource price, relaxed suffix value, and prefix upper bound.\\
\bottomrule
\end{tabularx}
\end{table}

\subsection{Gaussian identity, including singular sources}
\label{app:gaussian-identity}

Fix $W\in\mathbb R^{m\times n}$, $\Sigma_x\succeq0$, and $\Sigma_z\succ0$. Let
\begin{equation}
 x\sim\mathcal N(0,\Sigma_x),\qquad z\sim\mathcal N(0,\Sigma_z),\qquad
 x\perp z,\qquad y=Wx+z.
 \label{eq:channel}
\end{equation}
The output covariances are $\operatorname{Cov}(y\mid W)=W\Sigma_xW^\top+\Sigma_z$ and $\operatorname{Cov}(y\mid x,W)=\Sigma_z$. Subtracting their Gaussian differential entropies gives \citep{cover2006elements,telatar1999capacity}
\begin{equation}
 \begin{aligned}
 2I(x;y\mid W)
 &=\ln\det(W\Sigma_xW^\top+\Sigma_z)-\ln\det(\Sigma_z)\\
 &=\ln\det(I_m+\Sigma_z^{-1/2}W\Sigma_xW^\top\Sigma_z^{-1/2})\\
 &=\ln\det(I_n+\Sigma_x^{1/2}W^\top\Sigma_z^{-1}W\Sigma_x^{1/2}).
 \end{aligned}
 \label{eq:gaussian-proof}
\end{equation}
The last equality is Sylvester's determinant identity. The argument remains valid for singular $\Sigma_x$ because both output covariances are positive definite. The quantity is twice the mutual information for the prescribed Gaussian source.

For a random reference projection, we average this pointwise expression over its ensemble. For a write projection, the source covariance is itself generated by $W_b^{\mathrm{in}}$. The identity is then applied conditional on these factors before taking their expectation.

\section{The residual reference and its induced metrics}
\label{app:reference}

\subsection{Residual windows and a process realizing the ledger}
\label{app:residual-process}

Let $A_u\in\mathbb R^{u\times M}$ select the first $u$ residual coordinates. A branch with active residual width $\widetilde d_b$ has the actual residual form
\begin{equation}
 h_{b+1}=h_b+A_{\widetilde d_b}^\top F_b\!\left(\operatorname{Norm}(A_{\widetilde d_b}h_b)\right).
 \label{eq:residualupdate}
\end{equation}
This parameter-free selection and insertion preserves coordinates outside the write window, as also used in VWT \citep{wu2026variablewidth}. Both branches of block $\ell$ read and write its first $d_\ell$ coordinates.

To construct an architecture-only reference, introduce an auxiliary process $\widetilde h_b$:
\begin{equation}
 \begin{aligned}
 \widetilde h_0&\sim\mathcal N\!\left(0,\operatorname{diag}_{i\le M}\mathbf1\{i\le E\}\right),\\
 \widetilde h_{b+1}&=\widetilde h_b+A_{\widetilde d_b}^\top\varepsilon_b,
 &\varepsilon_b&\sim\mathcal N(0,I_{\widetilde d_b}).
 \end{aligned}
 \label{eq:reference-process}
\end{equation}
The initial state and all increments are independent. Consequently,
\begin{equation}
 \operatorname{Cov}(\widetilde h_b)=\operatorname{diag}(p_b),\qquad
 \operatorname{Cov}(\widetilde h_{b+1})
 =\operatorname{diag}(p_b)+A_{\widetilde d_b}^\top A_{\widetilde d_b},
 \label{eq:reference-covariance}
\end{equation}
which gives Equation~\eqref{eq:ledger}.

An intermediate FFN width changes neither the support nor the unit variance assigned to a reference write. The ledger therefore depends on previous residual-write windows, while FFN width enters the projection capacities and parameter costs. Within a search specification, we zero-pad ledgers to the maximum width allowed by its menu.

\subsection{Source normalization and inverse-background evaluation}
\label{app:source-metrics}

On a read window of width $\widetilde d_b$, the expected squared norm of the auxiliary state is $\widetilde d_b\bar p_b$. For the appendix derivations, denote the normalized branch-input covariance by $\Sigma_{x,b}^{\mathrm{in}}$. Replacing the sample root-mean-square normalization factor by $\sqrt{\bar p_b}$ yields
\begin{equation}
 \operatorname{Cov}\!\left(\frac{A_{\widetilde d_b}\widetilde h_b}{\sqrt{\bar p_b}}\right)
 =\frac{\operatorname{diag}(p_{b,1:\widetilde d_b})}{\bar p_b}=\Sigma_{x,b}^{\mathrm{in}},
 \qquad \bar p_b>0.
 \label{eq:source-covariance}
\end{equation}
This gives the mean-field normalized covariance. When $\bar p_b=0$, all source coordinates have zero reference power and we define $\Sigma_{x,b}^{\mathrm{in}}=0$.

For evaluation, we weight by the inverse of the sum of accumulated background and unit noise. Because all involved matrices are diagonal, the reference covariances in Equation~\eqref{eq:metrics} are equivalently
\begin{equation}
 \begin{aligned}
 \Sigma_{x,b}^{\mathrm{read}}&=Z_L (\Sigma_{x,b}^{\mathrm{in}})^{1/2}
       \bigl(I_{\widetilde d_b}+\operatorname{diag}(p_{b,1:\widetilde d_b})\bigr)^{-1}(\Sigma_{x,b}^{\mathrm{in}})^{1/2},\\
 \Sigma_{z,b}^{\mathrm{write}}&=Z_L^{-1}\bigl(I_{\widetilde d_b}+\operatorname{diag}(p_{b,1:\widetilde d_b})\bigr).
 \end{aligned}
 \label{eq:metrics-matrix-form}
\end{equation}
The branch propagates $\Sigma_{x,b}^{\mathrm{in}}$; the read covariance and write noise covariance apply inverse-background weighting during evaluation.

\subsection{Depth normalization and the scalar specialization}
\label{app:scalar-reference}

For the full-width reference with $\widetilde d_b=E$ for every branch, every coordinate in the window has $p_{b,i}=b+1$. The unnormalized receiver precisions are $(b+2)^{-1}$, so setting their mean to one gives
\begin{equation}
 \begin{aligned}
 Z_L&=\frac{2L}{\displaystyle\sum_{b=0}^{2L-1}(b+2)^{-1}},\qquad
 \Sigma_{x,b}^{\mathrm{in}}=I_E,\\
 \Sigma_{x,b}^{\mathrm{read}}&=(\Sigma_{z,b}^{\mathrm{write}})^{-1}=\frac{Z_L}{b+2}I_E.
 \end{aligned}
 \label{eq:isotropic}
\end{equation}
Every candidate of depth $L$ uses the same $Z_L$. In block $\ell$, the attention and FFN precisions are $Z_L/(2\ell)$ and $Z_L/(2\ell+1)$, respectively, reflecting the intervening attention write.

The harmonic form follows from unit increments and inverse-background evaluation. For a fixed active width different from $E$, the ledger is likewise depth-determined and is obtained directly from Equation~\eqref{eq:ledger}.

\subsection{Source support and the mean propagated covariance}
\label{app:source-support}

To expose the upstream ensemble in the appendix, write $W_b^{\mathrm{in}}=G_bX_b$, with normalized Gaussian factor $X_b\in\mathbb R^{k_b\times \widetilde d_b}$ and real diagonal gate $G_b$. Here $k_b$ is the input width of the scored output projection. Attention uses $k_b=\widetilde d_b$ and $G_b=I$; the FFN uses $k_b=f_\ell$ and the gate specified in Appendix~\ref{app:gate-reference}. Thus Equation~\eqref{eq:latent} is $\Sigma_{x,b}^{\mathrm{write}}=G_bX_b\Sigma_{x,b}^{\mathrm{in}}X_b^\top G_b$.

The diagonal factors in Equation~\eqref{eq:metrics-matrix-form} are positive on the source support, so $\Sigma_{x,b}^{\mathrm{read}}$ and $\Sigma_{x,b}^{\mathrm{in}}$ have the same rank. If $r=\operatorname{rank}\Sigma_{x,b}^{\mathrm{in}}$, then
\begin{equation}
 \operatorname{rank}(W\Sigma_{x,b}^{\mathrm{read}}W^\top)\le r,\qquad
 \operatorname{rank}(G_bX_b\Sigma_{x,b}^{\mathrm{in}}X_b^\top G_b)\le\min(k_b,r).
 \label{eq:source-rank}
\end{equation}
Expanding the first read window beyond $E$ therefore leaves its source rank at most $E$. A wider output can distribute features over a larger destination without increasing the current input rank. Subsequent reference writes update the ledger through Equation~\eqref{eq:ledger}.

For nonzero $\Sigma_{x,b}^{\mathrm{in}}$, $\operatorname{tr}\Sigma_{x,b}^{\mathrm{in}}=\widetilde d_b$. Independent Gaussian entries of $X_b$ with variance $1/\widetilde d_b$ satisfy
\begin{equation}
 \begin{aligned}
 \mathbb E_{X_b}[X_b\Sigma_{x,b}^{\mathrm{in}}X_b^\top]
 &=\frac{\operatorname{tr}\Sigma_{x,b}^{\mathrm{in}}}{\widetilde d_b}I_{k_b}=I_{k_b},\\
 \mathbb E_{X_b,G_b}[\Sigma_{x,b}^{\mathrm{write}}]&=\mathbb E_{G_b}[G_b I_{k_b}G_b]=I_{k_b},
 \end{aligned}
 \label{eq:latent-mean}
\end{equation}
where the second line uses independence and unit second moment of every gate entry. When $\Sigma_{x,b}^{\mathrm{in}}=0$, the propagated covariance is zero instead. ResNSC evaluates each propagated covariance inside the log-determinant before averaging.

\section{Projection-specific reference ensembles}
\label{app:projections}
\label{app:heads}

This appendix specifies the reference ensembles for the common capacity functional. For local formulas we suppress the branch subscript on $\Sigma_{x,b}^{\mathrm{in}},\Sigma_{x,b}^{\mathrm{read}},\Sigma_{z,b}^{\mathrm{write}},X_b,G_b$ and write the scored projection as $W$. In a write term, $W$ denotes the output projection and $X$ its normalized upstream reference.

\subsection{Complete reads and inherited entry variances}
\label{app:read-ensemble}

For $W\in\mathbb R^{m\times d}$ with independent $\mathcal N(0,v)$ entries, the read-ensemble evaluation is
\begin{equation}
 \psi^{\mathrm{read}}_v(\Sigma_x^{\mathrm{read}};m)
 :=\mathbb E_W\psi(W;\Sigma_x^{\mathrm{read}},I_m)
 =\mathbb E_W\ln\det(I_m+W\Sigma_x^{\mathrm{read}}W^\top).
 \label{eq:readcapacity}
\end{equation}
Its deterministic evaluation is denoted $\widehat\psi^{\mathrm{read}}_v(\Sigma_x^{\mathrm{read}};m)$. Table~\ref{tab:variance-conventions} gives the reference variances for all seven projections.

\begin{table}[htbp]
\centering
\caption{Projection and propagation conventions for a standard block, with $d=d_\ell$ and $f=f_\ell$.}
\label{tab:variance-conventions}
\small
\renewcommand{\arraystretch}{1.2}
\begin{tabular}{@{}llll@{}}
\toprule
Projection & Scored matrix shape & Entry variance & Upstream reference for a write\\
\midrule
Q, K, V & $d\times d$ & $2/(d+d/H)$ & None\\
O & $d\times d$ & $1/d$ & $k=d$, $\operatorname{Var}(X_{ij})=1/d$, $G=I_d$\\
gate, up & $f\times d$ & $2/(d+f)$ & None\\
down & $d\times f$ & $2/(d+f)$ & $k=f$, $\operatorname{Var}(X_{ij})=1/d$, normalized gate\\
\bottomrule
\end{tabular}
\end{table}

For attention, each complete Q, K, or V matrix vertically stacks $H$ head matrices of shape $(d/H)\times d$. Each stacked projection retains the per-head entry variance $2/(d+d/H)$ and is evaluated as one complete observation. The number of heads and admissible widths must make $d/H$ an integer in this standard attention setting.

\paragraph{Joint observation of a shared source.}
Let $x\sim\mathcal N(0,\Sigma_x^{\mathrm{read}})$ and let head $h$ observe $y_h=W_hx+z_h$, with mutually independent unit Gaussian receiver noises, independent of $x$. Define $K_h=(\Sigma_x^{\mathrm{read}})^{1/2}W_h^\top W_h(\Sigma_x^{\mathrm{read}})^{1/2}\succeq0$. Stacking the observations gives
\begin{equation}
 \psi^{\mathrm{joint}}=\ln\det\!\left(I_d+\sum_{h=1}^{H}K_h\right),\qquad
 \psi^{\mathrm{separate}}=\sum_{h=1}^{H}\ln\det(I_d+K_h).
 \label{eq:jointheads}
\end{equation}
For positive semidefinite $K_1,K_2$,
\begin{equation}
 \begin{aligned}
 \ln\det(I+K_1+K_2)
 &=\ln\det(I+K_1)
   +\ln\det\!\left(I+K_2^{1/2}(I+K_1)^{-1}K_2^{1/2}\right)\\
 &\le\ln\det(I+K_1)+\ln\det(I+K_2).
 \end{aligned}
 \label{eq:head-inequality}
\end{equation}
The equality follows from determinant factorization and Sylvester's identity; the inequality uses $(I+K_1)^{-1}\preceq I$ and monotonicity of the log-determinant. Induction gives $\psi^{\mathrm{joint}}\le\psi^{\mathrm{separate}}$, pointwise and hence in expectation under matched ensembles. Exact recovery of an original NSC implementation uses its projection decomposition and variance conventions together with isotropic local metrics.

\subsection{Normalized propagation and the gated FFN reference}
\label{app:gate-reference}

A write projection $W\in\mathbb R^{d\times k}$ receives the source $\Sigma_x^{\mathrm{write}}=GX\Sigma_x^{\mathrm{in}}X^\top G$. Its normalized upstream factor has shape $X\in\mathbb R^{k\times d}$ and independent entries with variance $v_{\mathrm{in}}=1/d$. The scored matrix has independent entries with variance $v_{\mathrm{out}}=2/(d+k)$. The factors $X,G,W$ are independent in this reference.

For attention, $k=d$ and $G=I_d$ define a normalized value-feature-to-output reference. For the gated FFN, $k=f$ and the independent diagonal entries of $G$ follow the normalized SiLU law
\begin{equation}
 \eta=\frac{\operatorname{SiLU}(\sqrt q\,z)}
 {\sqrt{\mathbb E_{z\sim\mathcal N(0,1)}[\operatorname{SiLU}(\sqrt q\,z)^2]}},
 \quad z\sim\mathcal N(0,1),\quad q=\frac{2d}{d+f},\quad
 \mu_{d,f}=\mathcal L(\eta^2).
 \label{eq:gate}
\end{equation}
The mean-field preactivation variance is $q=\frac{2}{d+f}\operatorname{tr}\Sigma_x^{\mathrm{in}}=2d/(d+f)$ for a nonzero source. Normalization gives $\mathbb E\eta^2=1$. The independent-gate approximation retains this marginal gate variation, and the real diagonal $G$ preserves positive semidefiniteness of $GX\Sigma_x^{\mathrm{in}}X^\top G$.

The expected write capacity is
\begin{equation}
 \begin{aligned}
 &\psi^{\mathrm{write}}(\Sigma_x^{\mathrm{in}},\Sigma_z^{\mathrm{write}};k,v_{\mathrm{in}},v_{\mathrm{out}},\mu)\\
 &\qquad:=\mathbb E_{W,X,G}\psi(W;GX\Sigma_x^{\mathrm{in}}X^\top G,\Sigma_z^{\mathrm{write}})\\
 &\qquad=\mathbb E_{W,X,G}\ln\det\!\left(
 I_d+(\Sigma_z^{\mathrm{write}})^{-1/2}WGX\Sigma_x^{\mathrm{in}}X^\top GW^\top (\Sigma_z^{\mathrm{write}})^{-1/2}\right).
 \end{aligned}
 \label{eq:writecapacity}
\end{equation}
Here $\mu$ is the law of a squared gate entry, and the input and output residual widths are both $d$, while the intermediate width is $k$. Its deterministic evaluation is $\widehat\psi^{\mathrm{write}}$. For a zero source covariance both versions are zero. The Gaussian-channel identity applies conditional on $X,G,W$, preserving the propagated spectrum inside the log-determinant before averaging.

\subsection{The complete block sum}
\label{app:block-score}

At the attention branch of block $\ell$, write $b=2\ell-2$, $d=d_\ell$, and $f=f_\ell$. The next ledger is $p_{b+1}=p_b+\mathbf1_{[1,d]}$, where $\mathbf1_{[1,d]}$ is an $M$-dimensional indicator of the active prefix. The covariance and metrics for the FFN are recomputed from $p_{b+1}$ using the same read width $d$. Let $v_A=2/(d+d/H)$ and $v_F=2/(d+f)$. Then the block shorthand in Equation~\eqref{eq:architecturecapacity} expands as
\begin{equation}
 \begin{aligned}
 S_L^{\mathrm{block}}(p_b;d,f)
 ={}&3\widehat\psi^{\mathrm{read}}_{v_A}(\Sigma_{x,b}^{\mathrm{read}};d)\\
   &+\widehat\psi^{\mathrm{write}}(\Sigma_{x,b}^{\mathrm{in}},\Sigma_{z,b}^{\mathrm{write}};d,1/d,1/d,\delta_1)\\
   &+2\widehat\psi^{\mathrm{read}}_{v_F}(\Sigma_{x,b+1}^{\mathrm{read}};f)\\
   &+\widehat\psi^{\mathrm{write}}(\Sigma_{x,b+1}^{\mathrm{in}},\Sigma_{z,b+1}^{\mathrm{write}};f,1/d,v_F,\mu_{d,f}).
 \end{aligned}
 \label{eq:block-expanded}
\end{equation}
The coefficients count the three attention reads and two FFN reads; the remaining terms score O and down. After the FFN write, $p_{b+2}=p_b+2\mathbf1_{[1,d]}$.

The dependence on $L$ enters through $Z_L$; $E,H$, the gate choice, and variance conventions are fixed by the specification.

\section{Deterministic evaluation of local capacity}
\label{app:spectral}

The expectations in Equations~\eqref{eq:readcapacity} and~\eqref{eq:writecapacity} are evaluated using deterministic equivalents for Gaussian covariance and product ensembles \citep{hoydis2011iterative}. These approximate the expected log-determinants at finite dimensions. Their deterministic values, with the declared gate quadrature, define the numerical objective optimized by ResNSC-DP.

The large-system interpretation scales the participating dimensions proportionally while retaining their aspect ratios and using normalized variances and controlled covariance spectra. The scalar systems below take the reference-metric eigenvalues as inputs.

\subsection{Single-projection evaluation}
\label{app:single-spectrum}\label{app:single}

Let $\tau_1,\ldots,\tau_d$ be the eigenvalues of $\Sigma_x^{\mathrm{read}}\succeq0$ and let $W\in\mathbb R^{m\times d}$ have entry variance $v>0$. Solve
\begin{equation}
 1-\xi=\sum_{i=1}^{d}\frac{v\xi\tau_i}{1+mv\xi\tau_i},
 \qquad 0<\xi\le1.
 \label{eq:single-fixedpoint}
\end{equation}
The left side strictly decreases from one to zero, while the right side is continuous and nondecreasing. There is therefore a unique root; if all $\tau_i$ vanish, it is $\xi=1$. The deterministic read evaluation is
\begin{equation}
 \widehat\psi^{\mathrm{read}}_v(\Sigma_x^{\mathrm{read}};m)
 =\sum_{i=1}^{d}\ln(1+mv\xi\tau_i)-m\ln\xi-m(1-\xi).
 \label{eq:single-evaluation}
\end{equation}
For $\Sigma_x^{\mathrm{read}}=\tau I_d$, this reduces to the Marchenko--Pastur Shannon transform used by NSC with effective entry variance $v\tau$ \citep{marchenko1967distribution,zhu2026nsc}. For nonuniform histories, the entire eigenvalue list contributes. We solve Equation~\eqref{eq:single-fixedpoint} by bisection; a zero metric gives exactly zero capacity.

\subsection{Propagated write evaluation}
\label{app:product-spectrum}\label{app:product}

Let $\chi_1,\ldots,\chi_d$ be the eigenvalues of $\Sigma_x^{\mathrm{in}}\succeq0$ and $\rho_1,\ldots,\rho_d$ the precision eigenvalues of $(\Sigma_z^{\mathrm{write}})^{-1}\succ0$. Let $\mu$ be the squared-gate law. For a nonzero source, the coupled fixed-point system is
\begin{equation}
 \begin{aligned}
 \gamma&=\int\frac{u}{1+\alpha\beta u}\,d\mu(u),\\
 \alpha&=v_{\mathrm{in}}\sum_{i=1}^{d}
       \frac{\chi_i}{1+kv_{\mathrm{in}}\beta\gamma\chi_i},\qquad
 \beta=v_{\mathrm{out}}\sum_{j=1}^{d}
       \frac{\rho_j}{1+kv_{\mathrm{out}}\alpha\gamma\rho_j}.
 \end{aligned}
 \label{eq:product-fixedpoint}
\end{equation}
The associated deterministic write evaluation is
\begin{equation}
 \begin{aligned}
 &\widehat\psi^{\mathrm{write}}(\Sigma_x^{\mathrm{in}},\Sigma_z^{\mathrm{write}};k,v_{\mathrm{in}},v_{\mathrm{out}},\mu)\\
 &\quad=\sum_{i=1}^{d}\ln(1+kv_{\mathrm{in}}\beta\gamma\chi_i)
       +\sum_{j=1}^{d}\ln(1+kv_{\mathrm{out}}\alpha\gamma\rho_j)\\
 &\qquad\quad+k\int\ln(1+\alpha\beta u)\,d\mu(u)-2k\alpha\beta\gamma.
 \end{aligned}
 \label{eq:product-evaluation}
\end{equation}
These equations specialize the single-user double-scattering system in \citet[Theorems~4--5]{hoydis2011iterative}. In that paper's notation, set $N=n_1=d$, $N_1=k$, and the signal-to-noise ratio to one; take $R_1=kv_{\mathrm{out}}(\Sigma_z^{\mathrm{write}})^{-1}$, $T_1=dv_{\mathrm{in}}\Sigma_x^{\mathrm{in}}$, $Q_1=I_d$, and $S_1$ as the diagonal squared-gate matrix. Substituting $\bar g_1=\alpha$, $\delta_1=\beta$, and $g_1=(k/d)\beta\gamma$, and expressing the normalized gate trace as an integral against $\mu$, gives Equation~\eqref{eq:product-fixedpoint}. Multiplying the normalized log-determinant in their Theorem~5 by $d$ gives Equation~\eqref{eq:product-evaluation}, including the term $-2k\alpha\beta\gamma$.

We use the positive solution specified by this system. The cited complex-Gaussian result supplies the first-order spectral approximation used here for the real Gaussian reference ensembles. The numerical gate law below has finite support, consistent with the bounded-spectrum large-system setting. A zero source covariance gives zero capacity directly.

For attention, $\mu=\delta_1$. For the gated FFN, we use 32-point Gauss--Hermite quadrature for the standard-normal variable in Equation~\eqref{eq:gate}. Given nodes $z_s$ and normalized weights $\omega_s$, define
\begin{equation}
 u_s=\frac{\operatorname{SiLU}(\sqrt q\,z_s)^2}
 {\displaystyle\sum_{v=1}^{32}\omega_v\operatorname{SiLU}(\sqrt q\,z_v)^2},
 \qquad
 \mu_{d,f}^{(32)}=\sum_{s=1}^{32}\omega_s\delta_{u_s}.
 \label{eq:gate-quadrature}
\end{equation}
Then $\sum_s\omega_su_s=1$. All candidates use this discrete law in the numerical block score.

The coupled equations are evaluated by damped fixed-point iteration, initialized at $\alpha=v_{\mathrm{in}}\operatorname{tr}\Sigma_x^{\mathrm{in}}$ and $\beta=v_{\mathrm{out}}\operatorname{tr}[(\Sigma_z^{\mathrm{write}})^{-1}]$. Repeated eigenvalues are grouped with their multiplicities in every sum in Equations~\eqref{eq:single-fixedpoint}--\eqref{eq:product-evaluation}.

\paragraph{Numerical settings and finite-matrix validation.}
Single-projection evaluation uses 48 bisection steps. Product evaluation uses damping $1/2$, stopping tolerance $2\times10^{-12}$ on the relative fixed-point update, and a maximum of 1,000 iterations. All computations use float64 and the 32-point gate rule above.

Table~\ref{tab:spectral-mc-validation} compares these evaluations with explicit independent Gaussian reference matrices at $d\in\{64,128\}$. The isotropic history is $p_i=5$; the nonuniform history assigns 5 and 1 to equal coordinate groups; the zero-source case assigns 5 and 0. We use $L=8$, construct $\Sigma_x^{\mathrm{in}}$ by Equation~\eqref{eq:source-covariance} and $\Sigma_x^{\mathrm{read}},\Sigma_z^{\mathrm{write}}$ by Equation~\eqref{eq:metrics}, and set $k=d$ for these three cases. The gated case uses the nonuniform history and $k=2d$, drawing independent gates from the normalized 32-point law in Equation~\eqref{eq:gate-quadrature}. Each case uses 1,200 independent draws, upstream variance $1/d$, write variance $2/(d+k)$, and read variance $2/(d+k)$. The maximum relative differences round to 0.337\% for reads and 0.686\% for writes in these cases.
\begin{table}[!htbp]
\centering
\caption{Finite-matrix validation: absolute relative difference between the deterministic value and the mean of 1,200 independent reference draws (\%).}
\label{tab:spectral-mc-validation}
\small
\begin{tabular}{@{}lrrr@{}}
\toprule
Context & $d$ & Read difference (\%) & Write difference (\%) \\
\midrule
Isotropic & 64 & 0.176 & 0.444 \\
Nonuniform history & 64 & 0.222 & 0.397 \\
Zero source directions & 64 & 0.337 & 0.647 \\
SiLU-gated & 64 & 0.165 & 0.562 \\
Isotropic & 128 & 0.135 & 0.203 \\
Nonuniform history & 128 & 0.135 & 0.260 \\
Zero source directions & 128 & 0.132 & 0.337 \\
SiLU-gated & 128 & 0.047 & 0.686 \\
\bottomrule
\end{tabular}
\end{table}

\section{History-aware search, bounds, and specialization}
\label{app:search}

This appendix supplies the cost convention, state argument, resource-price bounds, and complete search procedure. All results concern the finite product menus in Equation~\eqref{eq:searchproblem}, with two equal-window writes per block.

\subsection{Search domain and parameter costs}
\label{app:search-domain}

A specification $g$ fixes $L_g,E_g,H_g$, a residual-width menu, admissible $(d,f)$ pairs in $\mathcal A_g$, and reference conventions. Each block chooses a pair from that menu. The depth and chosen pair determine the block cost $c_{g,\ell}(d,f)$; the specification determines the fixed cost $c_{\mathrm{fixed}}(g)$. Define the variable budget
\begin{equation}
 \mathcal B_g=B_g-c_{\mathrm{fixed}}(g).
 \label{eq:variable-budget}
\end{equation}
For bias-free standard attention and gated FFNs, two width-$d$ block norms, and a final norm of width equal to the last active width, the parameter cost is
\begin{equation}
 c_{g,\ell}(d,f)=4d^2+3df+2d+\mathbf1\{\ell=L_g\}d.
 \label{eq:parameter-cost}
\end{equation}
The terminal term counts the final width-$d_L$ RMSNorm, and $c_{\mathrm{fixed}}=2VE$ counts untied embedding and readout matrices. Appendix~\ref{app:widths} gives the instantiated parameter counts.

For prefixes within $g$, $\mathbf n$ counts the residual-width choices, $t=\sum_rn_r$, and $p_g(\mathbf n)$ is Equation~\eqref{eq:histogram} with that specification's embedding support and width menu. A local choice $(d,f)$ has score $S_{L_g}^{\mathrm{block}}(p_g(\mathbf n);d,f)$ with all other reference arguments fixed by $g$.

\subsection{Order dependence and sufficient history}
\label{app:history-proof}

For two choices $(d_1,f_1)$ and $(d_2,f_2)$ evaluated after a ledger $p$, the score difference on exchanging them is
\begin{equation}
 \begin{aligned}
 &S_L^{\mathrm{block}}(p;d_1,f_1)
 +S_L^{\mathrm{block}}(p+2\mathbf1_{[1,d_1]};d_2,f_2)\\
 &\qquad-S_L^{\mathrm{block}}(p;d_2,f_2)
 -S_L^{\mathrm{block}}(p+2\mathbf1_{[1,d_2]};d_1,f_1).
 \end{aligned}
 \label{eq:exchange}
\end{equation}
Both orders end at $p+2\mathbf1_{[1,d_1]}+2\mathbf1_{[1,d_2]}$. Their intermediate contexts differ, allowing different accumulated scores.

\begin{proposition}[Residual-history sufficiency and dominance]
\label{prop:sufficiency}
Fix $g$ and the two-write reference. Prefixes with the same count vector $\mathbf n$ have identical additional scores for every common suffix. If their remaining budgets agree, their feasible continuation sets agree. If their spent costs and accumulated scores satisfy $c_1\le c_2$ and $s_1\ge s_2$, respectively, then the first prefix dominates the second.
\end{proposition}

\begin{proof}
Equation~\eqref{eq:histogram} reconstructs the ledger from $\mathbf n$, and $t=\sum_rn_r$ fixes the next block index. For a common next choice $(d,f)$, both prefixes therefore use identical reference contexts, contribute the same value $S_{L_g}^{\mathrm{block}}(p_g(\mathbf n);d,f)$, and enter the same next count state. Induction over the suffix length gives equality of all additional scores. Feasible continuations depend only on the fixed specification, block index, choices, and remaining additive budget. Any suffix affordable after cost $c_2$ is affordable after $c_1\le c_2$. Adding its common score to $s_1\ge s_2$ proves dominance.
\end{proof}

The result applies to the stated product menus and additive budget. Residual widths determine the unit-write increments; $c$ and $s$ retain the costs and earned scores of all choices, including FFN widths.

\subsection{Resource-price suffix bounds}
\label{app:resource-bound}

A prefix label is $\pi=(g,\mathbf n,c,s)$. Let $r(d)$ be the index of width $d$ in the menu, and let $\mathbf e_r$ be the corresponding unit count vector. For a nonnegative resource price $\lambda$, define
\begin{equation}
 \begin{aligned}
 V_{g,\lambda}(\mathbf n)
 =\max_{(d,f)\in\mathcal A_g}\bigl\{
 &S_{L_g}^{\mathrm{block}}(p_g(\mathbf n);d,f)-\lambda c_{g,t+1}(d,f)\\
 &+V_{g,\lambda}(\mathbf n+\mathbf e_{r(d)})\bigr\},
 \qquad V_{g,\lambda}(\mathbf n)=0\ \text{if }t=L_g.
 \end{aligned}
 \label{eq:potential}
\end{equation}
This backward recurrence retains history transitions but relaxes the hard remaining budget into a linear cost penalty. For a nonempty finite price set $\Lambda_g\subset[0,\infty)$, define
\begin{equation}
 U(\pi)=s+\min_{\lambda\in\Lambda_g}
 \left\{V_{g,\lambda}(\mathbf n)+\lambda(\mathcal B_g-c)\right\}.
 \label{eq:upperbound}
\end{equation}

\begin{proposition}[Admissible completion bound]
\label{prop:completion-bound}
When the suffix potentials in Equation~\eqref{eq:potential} are evaluated exactly, $U(\pi)$ bounds the score of every budget-feasible complete architecture extending $\pi$.
\end{proposition}

\begin{proof}
Take any feasible suffix with cost $c_{\mathrm{suf}}\le\mathcal B_g-c$ and additional score $s_{\mathrm{suf}}$. The relaxed recurrence maximizes over all such suffix choices and possibly budget-infeasible ones. Hence, for each $\lambda\ge0$,
\begin{equation}
 s_{\mathrm{suf}}-\lambda c_{\mathrm{suf}}\le V_{g,\lambda}(\mathbf n),
 \qquad
 s+s_{\mathrm{suf}}\le s+V_{g,\lambda}(\mathbf n)+\lambda(\mathcal B_g-c).
 \label{eq:bound-proof}
\end{equation}
Each price gives an upper bound on the same feasible completion; taking their minimum proves the claim. A conservatively overestimated potential also preserves the inequality.
\end{proof}

The state graph is acyclic because every transition increases $t$ by one. Potentials can therefore be computed backward. States known to lie on no budget-feasible root-to-terminal path may be removed, provided the remaining graph still contains every feasible completion relevant to the bound.

\subsection{ResNSC-DP and the global objective gap}
\label{app:solver}

ResNSC-DP maintains a maximum-priority queue $\mathcal Q$ of prefixes, nondominated frontiers indexed by $(g,\mathbf n)$, and an incumbent complete architecture $(\widehat g,\widehat a)$. Its score is $\mathrm{LB}=S_{\widehat g}(\widehat a)$. The initialization and search operations of Algorithm~\ref{alg:nsc-cert-full} are specified below.

\begin{algorithm}[t]
\caption{ResNSC-DP: detailed search procedure. Here $r(d)$ indexes residual width $d$ in its menu, and $\mathbf e_r$ is the corresponding unit count vector.}
\label{alg:nsc-cert-full}
\small
\begin{algorithmic}[1]
\Require Specifications $\mathcal G$, block menus $\mathcal A_g$, budgets $B_g$, tolerance $\epsilon\ge0$
\Ensure Architecture $(\widehat g,\widehat a)$ and bounds $(\mathrm{LB},\mathrm{UB})$, or infeasibility
\State $(\mathcal Q,\widehat g,\widehat a,\mathrm{LB})\gets\Call{Initialize}{\mathcal G,\{\mathcal A_g,B_g\}}$
\State \textbf{if} $\mathrm{LB}=-\infty$ \textbf{then return} infeasible
\While{$\mathcal Q\ne\varnothing$}
 \State $\mathcal Q\gets\Call{Prune}{\mathcal Q}$; $\mathrm{UB}\gets\max\{\mathrm{LB},\max_{\pi\in\mathcal Q}U(\pi)\}$
 \State \textbf{if} $\mathrm{UB}-\mathrm{LB}\le\epsilon$ or \Call{LimitReached}{} \textbf{then break}
 \State $\pi=(g,\mathbf n,c,s)\gets\Call{PopMax}{\mathcal Q}$; $t\gets\sum_r n_r$
 \For{$(d,f)\in\mathcal A_g$}
  \State $\mathbf n'\gets\mathbf n+\mathbf e_{r(d)}$, \quad $c'\gets c+c_{g,t+1}(d,f)$
  \State $s'\gets s+S_{L_g}^{\mathrm{block}}(p_g(\mathbf n);d,f)$, \quad $\pi'\gets(g,\mathbf n',c',s')$
  \State \textbf{if} $\neg\Call{Feasible}{\pi'}$ or \Call{Dominated}{$\pi'$} or $U(\pi')\le\mathrm{LB}$ \textbf{then continue}
  \State \textbf{if} $t+1=L_g$ \textbf{then} $(\widehat g,\widehat a,\mathrm{LB})\gets\Call{UpdateBest}{\pi'}$
  \State \textbf{else} \Call{Insert}{$\mathcal Q,\pi'$}
 \EndFor
\EndWhile
\State $\mathcal Q\gets\Call{Prune}{\mathcal Q}$; $\mathrm{UB}\gets\max\{\mathrm{LB},\max_{\pi\in\mathcal Q}U(\pi)\}$
\State \Return $(\widehat g,\widehat a),\mathrm{LB},\mathrm{UB}$ \Comment{$\max\varnothing=-\infty$}
\end{algorithmic}
\end{algorithm}

\paragraph{Initialization.}
\textsc{Initialize} prepares spectral values and resource-price potentials. Minimum-cost paths supply an incumbent for every feasible specification; price-solution paths and feasible one-block replacements improve it before feasible roots enter the queue.

\paragraph{Shared search operations.}
\textsc{Feasible} compares spent cost plus minimum completion cost with $\mathcal B_g$. \textsc{Dominated} and \textsc{Insert} maintain the frontier at $(g,\mathbf n)$ according to Proposition~\ref{prop:sufficiency}, marking displaced labels stale. \textsc{Prune} removes stale or incumbent-bounded entries, and \textsc{UpdateBest} retains the best feasible complete path. Time and expansion limits are checked between complete prefix expansions.

\begin{theorem}[Finite-menu objective certificate]
\label{thm:search}
Suppose Equation~\eqref{eq:searchproblem} has a nonempty feasible set and the suffix bounds are admissible. With exact arithmetic for the fixed objective, Algorithm~\ref{alg:nsc-cert-full} maintains the following bounds at initialization and after each complete prefix expansion:
\begin{equation}
 \mathrm{LB}=S_{\widehat g}(\widehat a)\le S^*\le
 \max\!\left\{\mathrm{LB},\max_{\pi\in\mathcal Q\text{ valid}}U(\pi)\right\}
 =\mathrm{UB}.
 \label{eq:global-gap}
\end{equation}
It returns an $\epsilon$-optimal architecture if $\mathrm{UB}-\mathrm{LB}\le\epsilon$, and an optimal architecture if the valid queue is exhausted. At an earlier stop, the same bounds quantify the remaining objective gap.
\end{theorem}

\begin{proof}
The feasible incumbent supplies a lower bound. Initially every feasible architecture has a prefix at one of the feasible roots. On expanding a valid prefix, all possible next choices are considered. Budget pruning removes only paths that cannot have a feasible completion. By Proposition~\ref{prop:sufficiency}, a dominated prefix can be replaced by one with no greater cost, at least as much earned score, and identical continuation values. By Proposition~\ref{prop:completion-bound}, a prefix with upper bound at most the incumbent cannot improve that incumbent.

After each complete expansion, every potentially improving feasible completion is therefore represented by a valid live prefix, or by a dominating prefix with at least as good a completion; completed paths have already updated the incumbent. The maximum live bound together with the incumbent bounds the optimum. The finite search tree is exhausted if no valid improving prefix remains, at which point $\mathrm{UB}=\mathrm{LB}$. Otherwise the gap bounds the incumbent's suboptimality, proving the stopping claim.
\end{proof}

\paragraph{Numerical objective and reporting.}
We evaluate the spectral objective in float64, with a small upward margin on computed completion bounds and absolute stopping tolerance $10^{-6}$. We report the absolute gap $\mathrm{UB}-\mathrm{LB}$ and relative gap $(\mathrm{UB}-\mathrm{LB})/\max(1,|\mathrm{LB}|)$; the search terminates when the valid queue is exhausted or the numerical gap reaches this tolerance. Theorem~\ref{thm:search} assumes exact arithmetic and admissible bounds.

\FloatBarrier
\subsection{Fixed-residual-width dynamic programming}
\label{app:fixed-width}

When $d_\ell=d$ for every layer, the ledger before each block is determined by depth independently of the FFN choices. In the full-width case $d=E$, $p_{2\ell-2,i}=2\ell-1$ on the active window, and the two scalar precisions follow Appendix~\ref{app:scalar-reference}. The block value remains $S_L^{\mathrm{block}}(p_{2\ell-2};d,f)$.

Let $D_\ell(B')$ be the highest score of the first $\ell$ blocks spending at most the integer variable budget $B'$. Fix one specification $g$ and write $L=L_g$. Then
\begin{equation}
 \begin{aligned}
 D_\ell(B')=\max_{\substack{f:(d,f)\in\mathcal A_g\\c_{g,\ell}(d,f)\le B'}}
 \bigl\{D_{\ell-1}(B'-c_{g,\ell}(d,f))
       +S_L^{\mathrm{block}}(p_{2\ell-2};d,f)\bigr\},\\
 D_0(B')=0\quad(B'\ge0),\qquad D_\ell(B')=-\infty\quad(B'<0).
 \end{aligned}
 \label{eq:fixed-width-dp}
\end{equation}
The maximum over an empty feasible set is $-\infty$. This multiple-choice knapsack recurrence is exact for the evaluated score and stated costs. The ledger remains depth-determined, preserving residual-context weighting and propagated write sources while making each layer's value independent of earlier FFN choices.

\subsection{Joint-width search results}
\label{app:search-measurements}

\paragraph{Joint-width search.}
On the complete 121M and 374M menus specified in Appendix~\ref{app:widths}, ResNSC-DP returns the reported allocations with zero final numerical objective gap (Table~\ref{tab:search-audit}). Both searches exhaust the valid queue.

\begin{table}[!htbp]
\centering
\caption{Final objectives and numerical gaps for joint residual/FFN-width search. Both searches exhaust the valid queue.}
\label{tab:search-audit}
\small\setlength{\tabcolsep}{4pt}
\begin{tabular}{@{}lrr@{}}
\toprule
Instance & Objective $S(\widehat a)$ & Rel. gap \\
\midrule
121M joint & 59,330.78 & 0 \\
374M joint & 133,930.49 & 0 \\
\bottomrule
\end{tabular}
\end{table}

\section{Training and evaluation statistics}
\label{app:training-statistics}

\paragraph{Data and training budget.}
Pretraining uses FineWeb-Edu \citep{penedo2024fineweb}, a 32,000-token Llama-2 tokenizer, and packed sequences of length 4,096. The fixed corpus subsets contain approximately 8.34B available tokens at 121M and 37.45B at 374M. The smaller subset uses the first ten of the larger subset's 45 source shards. Each subset is split by document index into 99.5\% training and 0.5\% validation. All architectures at a given scale use the same subset, split, and token budget.

The global batch contains 128 sequences, or 524,288 tokens. We set the main-phase budget to approximately 20 tokens per parameter using the uniform model's instantiated parameter count, following the budget convention motivated by compute-optimal training \citep{hoffmann2022chinchilla}. The main phase contains 4,616 updates at 121M and 14,267 at 374M, corresponding to approximately 2.42B and 7.48B tokens. A final linear-decay phase adds 462 and 1,427 updates. Total training therefore uses 2.662B and 8.228B tokens, or approximately 22 tokens per parameter.

\paragraph{Optimization.}
All models use AdamW \citep{loshchilov2019adamw} with base learning rate 0.01, $(\beta_1,\beta_2)=(0.9,0.95)$, weight decay 0.1, and numerical stabilizer $10^{-10}$. Initialization and parameter-group scaling follow the shared $\mu$P recipe \citep{yang2021tensorprogramsv}: initialization scale 0.1, width multiplier 8, embedding multiplier 12, and residual multiplier 0.22. Attention logits are scaled by the inverse head dimension. These are the training conventions; the reference variance rules defining the spectral score are specified separately in Section~\ref{sec:projections}.

The schedule uses linear warmup for the first 8\% of the main phase, followed by a power-law learning-rate schedule with coefficient 512 and exponent $-0.566$ as a function of consumed tokens. The final decay linearly reduces the ending main-phase learning rate to zero over approximately 10\% additional updates. Training uses eight A100-40GB GPUs, BF16, FSDP2\footnote{\url{https://docs.pytorch.org/docs/2.8/distributed.fsdp.fully_shard.html}}, and FlashAttention-2 \citep{dao2024flashattention2}.

\paragraph{Training seeds and aggregation.}
At 121M, every arm in Table~\ref{tab:rsc-pretraining} has three independent training runs (seeds 42, 43, 44), each starting from a fresh initialization and using its seed for data order; means and sample standard deviations are taken over these runs. At 374M and for the scoring ablations, each arm has one full run with seed 42 and two continuations of its final-decay phase (1,427 and 462 updates, respectively) that restore the seed-42 main checkpoint and regenerate the data order with seeds 43 and 44; the reported standard deviations measure annealing variation. Perplexity is exponentiated from each run's loss before averaging; Avg8 is computed within each run before aggregation. Table~\ref{tab:rsc-loss-appendix} lists the losses and Table~\ref{tab:rsc-340M-downstream} the 374M per-task accuracies.

\providecommand{\sd}[1]{{\scriptsize($\pm$#1)}}
\begin{table}[htbp]
\centering
\caption{Held-out FineWeb-Edu loss (nats) at 121M and 374M. 121M arms: three independent runs (seeds 42--44). 374M arms: one run (seed 42) plus two final-decay continuations, whose standard deviation reflects annealing variation only.}
\label{tab:rsc-loss-appendix}
\small
\setlength{\tabcolsep}{6pt}
\begin{tabular}{@{}l r r@{}}
\toprule
Method & 121M loss & 374M loss \\
\midrule
ResNSC & 2.96496\,\sd{0.00067} & 2.61025\,\sd{0.00017} \\
uniform & 3.00813\,\sd{0.00109} & 2.65818\,\sd{0.00011} \\
VWT & 2.96922\,\sd{0.00286} & 2.61367\,\sd{0.00016} \\
TLM & 3.01478\,\sd{0.00108} & 2.66964\,\sd{0.00006} \\
NSC & 2.99598\,\sd{0.00134} & 2.62640\,\sd{0.00019} \\
Reverse & 3.00905\,\sd{0.00233} & 2.67282\,\sd{0.00005} \\
\bottomrule
\end{tabular}
\end{table}

\begin{table}[htbp]
\centering
\caption{374M zero-shot accuracies (\%) after the final decay: eight tasks and their mean (Avg8), one run plus two final-decay continuations (mean and sample standard deviation). HellaSwag and ARC-c use normalized accuracy. Bold: best per column; underline: second best.}
\label{tab:rsc-340M-downstream}
\scriptsize
\setlength{\tabcolsep}{2pt}
\begin{tabularx}{\linewidth}{@{}l*{9}{>{\raggedleft\arraybackslash}X}@{}}
\toprule
Method & LMB $\uparrow$ & PIQA $\uparrow$ & HSwag $\uparrow$ & WGrnd $\uparrow$ & ARC-e $\uparrow$ & ARC-c $\uparrow$ & SIQA $\uparrow$ & BoolQ $\uparrow$ & Avg8 $\uparrow$ \\
\midrule
\rowcolor{rscblue} ResNSC (ours) & \shortstack[r]{\textbf{29.26}\\[-1pt]{\tiny $\pm$ 0.16}} & \shortstack[r]{\underline{64.89}\\[-1pt]{\tiny $\pm$ 0.33}} & \shortstack[r]{\underline{35.49}\\[-1pt]{\tiny $\pm$ 0.20}} & \shortstack[r]{50.41\\[-1pt]{\tiny $\pm$ 0.59}} & \shortstack[r]{\textbf{54.53}\\[-1pt]{\tiny $\pm$ 0.42}} & \shortstack[r]{25.06\\[-1pt]{\tiny $\pm$ 0.26}} & \shortstack[r]{\underline{37.17}\\[-1pt]{\tiny $\pm$ 0.16}} & \shortstack[r]{59.97\\[-1pt]{\tiny $\pm$ 0.96}} & \shortstack[r]{\textbf{44.60}\\[-1pt]{\tiny $\pm$ 0.07}} \\[2pt]
uniform & \shortstack[r]{26.69\\[-1pt]{\tiny $\pm$ 0.26}} & \shortstack[r]{64.56\\[-1pt]{\tiny $\pm$ 0.22}} & \shortstack[r]{34.16\\[-1pt]{\tiny $\pm$ 0.08}} & \shortstack[r]{\textbf{52.57}\\[-1pt]{\tiny $\pm$ 0.21}} & \shortstack[r]{52.46\\[-1pt]{\tiny $\pm$ 0.40}} & \shortstack[r]{\underline{26.11}\\[-1pt]{\tiny $\pm$ 0.34}} & \shortstack[r]{\textbf{37.22}\\[-1pt]{\tiny $\pm$ 0.18}} & \shortstack[r]{58.76\\[-1pt]{\tiny $\pm$ 0.36}} & \shortstack[r]{44.07\\[-1pt]{\tiny $\pm$ 0.19}} \\[2pt]
VWT & \shortstack[r]{27.03\\[-1pt]{\tiny $\pm$ 0.32}} & \shortstack[r]{\textbf{65.38}\\[-1pt]{\tiny $\pm$ 0.35}} & \shortstack[r]{\textbf{35.94}\\[-1pt]{\tiny $\pm$ 0.30}} & \shortstack[r]{50.51\\[-1pt]{\tiny $\pm$ 0.96}} & \shortstack[r]{\underline{53.68}\\[-1pt]{\tiny $\pm$ 0.76}} & \shortstack[r]{\textbf{26.39}\\[-1pt]{\tiny $\pm$ 0.47}} & \shortstack[r]{36.51\\[-1pt]{\tiny $\pm$ 0.34}} & \shortstack[r]{58.66\\[-1pt]{\tiny $\pm$ 0.55}} & \shortstack[r]{44.26\\[-1pt]{\tiny $\pm$ 0.11}} \\[2pt]
TLM & \shortstack[r]{26.24\\[-1pt]{\tiny $\pm$ 0.16}} & \shortstack[r]{63.13\\[-1pt]{\tiny $\pm$ 0.36}} & \shortstack[r]{33.87\\[-1pt]{\tiny $\pm$ 0.06}} & \shortstack[r]{50.75\\[-1pt]{\tiny $\pm$ 0.36}} & \shortstack[r]{52.57\\[-1pt]{\tiny $\pm$ 0.51}} & \shortstack[r]{25.63\\[-1pt]{\tiny $\pm$ 0.34}} & \shortstack[r]{36.18\\[-1pt]{\tiny $\pm$ 0.22}} & \shortstack[r]{\textbf{61.46}\\[-1pt]{\tiny $\pm$ 0.23}} & \shortstack[r]{43.73\\[-1pt]{\tiny $\pm$ 0.05}} \\[2pt]
NSC & \shortstack[r]{\underline{27.74}\\[-1pt]{\tiny $\pm$ 0.36}} & \shortstack[r]{64.42\\[-1pt]{\tiny $\pm$ 0.14}} & \shortstack[r]{35.27\\[-1pt]{\tiny $\pm$ 0.14}} & \shortstack[r]{50.33\\[-1pt]{\tiny $\pm$ 0.79}} & \shortstack[r]{53.04\\[-1pt]{\tiny $\pm$ 0.21}} & \shortstack[r]{25.74\\[-1pt]{\tiny $\pm$ 0.27}} & \shortstack[r]{36.51\\[-1pt]{\tiny $\pm$ 0.38}} & \shortstack[r]{\underline{61.07}\\[-1pt]{\tiny $\pm$ 0.27}} & \shortstack[r]{\underline{44.26}\\[-1pt]{\tiny $\pm$ 0.09}} \\[2pt]
Reverse & \shortstack[r]{25.76\\[-1pt]{\tiny $\pm$ 0.21}} & \shortstack[r]{64.16\\[-1pt]{\tiny $\pm$ 0.45}} & \shortstack[r]{33.62\\[-1pt]{\tiny $\pm$ 0.06}} & \shortstack[r]{\underline{52.07}\\[-1pt]{\tiny $\pm$ 0.84}} & \shortstack[r]{53.31\\[-1pt]{\tiny $\pm$ 0.48}} & \shortstack[r]{25.26\\[-1pt]{\tiny $\pm$ 0.09}} & \shortstack[r]{35.55\\[-1pt]{\tiny $\pm$ 0.06}} & \shortstack[r]{53.12\\[-1pt]{\tiny $\pm$ 1.72}} & \shortstack[r]{42.86\\[-1pt]{\tiny $\pm$ 0.18}} \\[2pt]
\bottomrule
\end{tabularx}
\end{table}

\paragraph{Evaluation protocol.}
FineWeb-Edu cross-entropy is measured on the same 2,048 held-out sequences across architectures and runs, each contributing 4,096 next-token targets. We use \texttt{lm-eval} 0.4.13 \citep{gao2023evalharness} with BF16 models, the shared Llama-2 tokenizer, a maximum input length of 4,096, and zero-shot task templates. Each task uses its complete evaluation split (Table~\ref{tab:eval-protocol}).

At both scales, we report WikiText-2 word perplexity and LAMBADA final-word perplexity. WikiText-2 uses rolling likelihood with word-count normalization; LAMBADA scores the final word conditioned on its preceding text. At 374M, the eight accuracy tasks are LAMBADA, PIQA, HellaSwag, WinoGrande, ARC-e, ARC-c, SocialIQA, and BoolQ. HellaSwag and ARC-c use normalized accuracy; the other tasks use accuracy. Avg8 is their arithmetic mean within each run.

\begin{table}[!htbp]
\centering
\caption{Pretraining-model evaluation datasets, splits, and metrics. All tasks use zero-shot templates and the complete indicated split.}
\label{tab:eval-protocol}
\small\setlength{\tabcolsep}{4pt}\renewcommand{\arraystretch}{1.12}
\begin{tabularx}{\linewidth}{@{}l>{\raggedright\arraybackslash}Xll@{}}
\toprule
Task ID & Dataset & Split & Reported metric \\
\midrule
\texttt{wikitext} & WikiText-2 \citep{merity2017wikitext} & test & word PPL \\
\texttt{lambada\_openai} & LAMBADA \citep{paperno2016lambada} & test & PPL, accuracy \\
\texttt{piqa} & PIQA \citep{bisk2020piqa} & validation & accuracy \\
\texttt{hellaswag} & HellaSwag \citep{zellers2019hellaswag} & validation & normalized accuracy \\
\texttt{winogrande} & WinoGrande (XL) \citep{sakaguchi2020winogrande} & validation & accuracy \\
\texttt{arc\_easy} & ARC-Easy \citep{clark2018arc} & test & accuracy \\
\texttt{arc\_challenge} & ARC-Challenge \citep{clark2018arc} & test & normalized accuracy \\
\texttt{social\_iqa} & SocialIQA \citep{sap2019socialiqa} & validation & accuracy \\
\texttt{boolq} & BoolQ \citep{clark2019boolq} & validation & accuracy \\
\bottomrule
\end{tabularx}
\end{table}

\subsection{Scoring ablation protocol}
\label{app:scoring-ablations}

Each scoring variant in Table~\ref{tab:rsc-component-ablation} selects an architecture from the same 121M width menus under the same parameter ceiling, then uses the training and evaluation protocols above. We test three simplifications of residual history. \emph{Without residual context} uses isotropic normalized source powers and identity read and receiver metrics, while retaining within-branch propagation. \emph{Depth only} sets every active coordinate's reference power to $p_{b,i}=b+1$, where $b$ indexes residual branches, replacing width history with branch position. \emph{Mean history only} uses the average reference power over the current active coordinates when evaluating each branch, retaining mean power while removing coordinate-wise differences.

To test propagated features, \emph{without propagated spectrum} replaces the write-source covariance $\Sigma_{x,b}^{\mathrm{write}}$ with the appropriately sized identity for both attention output and FFN down projections, preserving expected reference power while removing the propagated spectrum. Read scoring and the residual-dependent receiver covariance remain unchanged.

\section{Layer-wise architecture configurations}
\label{app:widths}

\paragraph{Architecture and allocation rules.}
Each Transformer block \citep{vaswani2017attention} contains pre-RMSNorm \citep{zhang2019rmsnorm} causal self-attention followed by a pre-RMSNorm SwiGLU FFN \citep{shazeer2020glu}. Attention uses rotary position embeddings \citep{su2021roformer} and 16 query and 16 key/value heads, with head dimension $d_\ell/16$. The FFN has gate and up projections from $d_\ell$ to $f_\ell$ and a down projection back to $d_\ell$. Linear projections are bias-free, and input embeddings and vocabulary readout are untied. Inactive residual coordinates retain their values along the identity path without additional parameters. The embedding/readout width is fixed at $E=640$ for 121M and $E=1024$ for 374M. Final RMSNorm acts on the last active width $d_L$, before the state is restored or cropped to readout width $E$.

We compare six allocation rules. \emph{Uniform} uses $d_\ell=E$ and $f_\ell=2.75E$. \emph{VWT} uses the default bottleneck-location and width-ratio settings $(0.75,0.3)$, with $f_\ell=4d_\ell$. \emph{TLM} holds residual width fixed and applies a cosine FFN taper from 1.5 to 0.5 times the uniform FFN width \citep{bayat2026tapered}. \emph{ResNSC} jointly selects active residual and FFN widths. \emph{NSC} maximizes the original NSC score over the same residual-width and FFN-width menu under the same parameter ceiling. \emph{Reverse} reverses the ResNSC sequence of $(d_\ell,f_\ell)$ pairs, preserving the projection-shape inventory used by the score.

\paragraph{Search menus.}
At each scale, the active-width menu is $d_\ell/E\in\{0.25,0.50,\ldots,3.00\}$. The FFN-ratio menu is $\{1.375,2.0625,2.75,3.4375,4.125,4.8125\}$ at 121M and $\{2.0625,2.75,3.4375,4.125,4.8125\}$ at 374M. FFN widths are rounded to multiples of 32. We search each scale independently under the uniform model's parameter ceiling, counting embeddings, readout, and normalization as well as projections. Each layer chooses its own active width and FFN ratio; the feasible set imposes no monotonicity constraint or prescribed taper.

Tables~\ref{tab:rsc-layers-121M} and~\ref{tab:rsc-layers-340M} give each model's exact $(d_\ell,f_\ell)$ sequence, ordered from input to output. These are active computation widths; the persistent residual state follows Equation~\eqref{eq:residualupdate}.

\paragraph{Parameter counts.}
Table~\ref{tab:parameter-count-audit} reports total parameters, including embeddings, readout, and normalization, using $2VE+\sum_\ell(4d_\ell^2+3d_\ell f_\ell+2d_\ell)+d_L$ with $V=32{,}000$. The last-active-width RMSNorm accounts for the ResNSC/Reverse differences of 640 parameters at 121M and 1,280 at 374M. The 121M VWT configuration exceeds the Uniform search ceiling by 12,960 parameters.

\begin{table}[!htbp]
\centering
\caption{Instantiated parameter counts for the pretraining configurations, including embeddings, readout, and normalization. Column headings are the scale labels used in the paper.}
\label{tab:parameter-count-audit}
\small
\setlength{\tabcolsep}{12pt}
\begin{tabular}{@{}lrr@{}}
\toprule
Method & 121M & 374M \\
\midrule
ResNSC & 121,188,640 & 373,799,936 \\
Uniform & 121,262,720 & 373,867,520 \\
VWT & 121,275,680 & 373,859,200 \\
TLM & 121,262,720 & 373,867,520 \\
NSC & 120,898,560 & 373,847,040 \\
Reverse & 121,189,280 & 373,801,216 \\
\bottomrule
\end{tabular}
\end{table}

\begin{table}[!htbp]
\centering
\caption{Exact 121M per-layer widths in input-to-output order: residual width $d_\ell$ and FFN intermediate width $f_\ell$.}
\label{tab:rsc-layers-121M}
\footnotesize
\setlength{\tabcolsep}{2pt}
\renewcommand{\arraystretch}{1.08}
\begin{tabularx}{\linewidth}{@{}r*{12}{>{\raggedleft\arraybackslash}X}@{}}
\toprule
Layer & \multicolumn{2}{c}{ResNSC} & \multicolumn{2}{c}{uniform} & \multicolumn{2}{c}{VWT} & \multicolumn{2}{c}{TLM} & \multicolumn{2}{c}{NSC} & \multicolumn{2}{c}{Reverse} \\
\cmidrule(lr){2-3}\cmidrule(lr){4-5}\cmidrule(lr){6-7}\cmidrule(lr){8-9}\cmidrule(lr){10-11}\cmidrule(lr){12-13}
 & $d$ & $f$ & $d$ & $f$ & $d$ & $f$ & $d$ & $f$ & $d$ & $f$ & $d$ & $f$ \\
\midrule
1 & 1120 & 1536 & 640 & 1760 & 864 & 3456 & 640 & 2624 & 640 & 1312 & 480 & 672 \\
2 & 1280 & 1760 & 640 & 1760 & 768 & 3072 & 640 & 2624 & 800 & 1088 & 480 & 672 \\
3 & 1280 & 1760 & 640 & 1760 & 704 & 2816 & 640 & 2560 & 800 & 1088 & 480 & 672 \\
4 & 960 & 1312 & 640 & 1760 & 608 & 2432 & 640 & 2464 & 800 & 1088 & 480 & 672 \\
5 & 960 & 1312 & 640 & 1760 & 544 & 2176 & 640 & 2336 & 800 & 1088 & 480 & 672 \\
6 & 800 & 1088 & 640 & 1760 & 512 & 2048 & 640 & 2208 & 800 & 1088 & 480 & 672 \\
7 & 640 & 896 & 640 & 1760 & 448 & 1792 & 640 & 2016 & 800 & 1088 & 480 & 672 \\
8 & 640 & 896 & 640 & 1760 & 416 & 1664 & 640 & 1856 & 800 & 1088 & 640 & 896 \\
9 & 640 & 896 & 640 & 1760 & 352 & 1408 & 640 & 1664 & 800 & 1088 & 640 & 896 \\
10 & 480 & 672 & 640 & 1760 & 320 & 1280 & 640 & 1504 & 800 & 1088 & 640 & 896 \\
11 & 480 & 672 & 640 & 1760 & 288 & 1152 & 640 & 1312 & 800 & 1088 & 800 & 1088 \\
12 & 480 & 672 & 640 & 1760 & 256 & 1024 & 640 & 1184 & 800 & 1088 & 960 & 1312 \\
13 & 480 & 672 & 640 & 1760 & 352 & 1408 & 640 & 1056 & 800 & 1088 & 960 & 1312 \\
14 & 480 & 672 & 640 & 1760 & 480 & 1920 & 640 & 960 & 800 & 1088 & 1280 & 1760 \\
15 & 480 & 672 & 640 & 1760 & 640 & 2560 & 640 & 896 & 800 & 1088 & 1280 & 1760 \\
16 & 480 & 672 & 640 & 1760 & 864 & 3456 & 640 & 896 & 640 & 896 & 1120 & 1536 \\
\bottomrule
\end{tabularx}
\end{table}

\begin{table}[!htbp]
\centering
\caption{Exact 374M per-layer widths in input-to-output order: residual width $d_\ell$ and FFN intermediate width $f_\ell$.}
\label{tab:rsc-layers-340M}
\footnotesize
\setlength{\tabcolsep}{2pt}
\renewcommand{\arraystretch}{1.08}
\begin{tabularx}{\linewidth}{@{}r*{12}{>{\raggedleft\arraybackslash}X}@{}}
\toprule
Layer & \multicolumn{2}{c}{ResNSC} & \multicolumn{2}{c}{uniform} & \multicolumn{2}{c}{VWT} & \multicolumn{2}{c}{TLM} & \multicolumn{2}{c}{NSC} & \multicolumn{2}{c}{Reverse} \\
\cmidrule(lr){2-3}\cmidrule(lr){4-5}\cmidrule(lr){6-7}\cmidrule(lr){8-9}\cmidrule(lr){10-11}\cmidrule(lr){12-13}
 & $d$ & $f$ & $d$ & $f$ & $d$ & $f$ & $d$ & $f$ & $d$ & $f$ & $d$ & $f$ \\
\midrule
1 & 1792 & 3712 & 1024 & 2816 & 1408 & 5632 & 1024 & 4224 & 1280 & 3520 & 512 & 1056 \\
2 & 2048 & 4224 & 1024 & 2816 & 1312 & 5248 & 1024 & 4224 & 1280 & 2656 & 512 & 1056 \\
3 & 2048 & 4224 & 1024 & 2816 & 1216 & 4864 & 1024 & 4160 & 1280 & 2656 & 512 & 1056 \\
4 & 2048 & 4224 & 1024 & 2816 & 1152 & 4608 & 1024 & 4096 & 1280 & 2656 & 512 & 1056 \\
5 & 1536 & 3168 & 1024 & 2816 & 1056 & 4224 & 1024 & 4032 & 1280 & 2656 & 512 & 1056 \\
6 & 1280 & 2656 & 1024 & 2816 & 992 & 3968 & 1024 & 3904 & 1280 & 2656 & 512 & 1056 \\
7 & 1280 & 2656 & 1024 & 2816 & 928 & 3712 & 1024 & 3776 & 1280 & 2656 & 512 & 1056 \\
8 & 1024 & 2112 & 1024 & 2816 & 864 & 3456 & 1024 & 3616 & 1280 & 2656 & 512 & 1056 \\
9 & 1024 & 2112 & 1024 & 2816 & 800 & 3200 & 1024 & 3456 & 1024 & 2112 & 768 & 1600 \\
10 & 1024 & 2112 & 1024 & 2816 & 736 & 2944 & 1024 & 3296 & 1024 & 2112 & 768 & 1600 \\
11 & 768 & 1600 & 1024 & 2816 & 704 & 2816 & 1024 & 3104 & 1024 & 2112 & 768 & 1600 \\
12 & 768 & 1600 & 1024 & 2816 & 640 & 2560 & 1024 & 2912 & 1024 & 2112 & 768 & 1600 \\
13 & 768 & 1600 & 1024 & 2816 & 608 & 2432 & 1024 & 2720 & 1024 & 2112 & 768 & 1600 \\
14 & 768 & 1600 & 1024 & 2816 & 544 & 2176 & 1024 & 2528 & 1024 & 2112 & 768 & 1600 \\
15 & 768 & 1600 & 1024 & 2816 & 512 & 2048 & 1024 & 2336 & 1024 & 2112 & 1024 & 2112 \\
16 & 768 & 1600 & 1024 & 2816 & 480 & 1920 & 1024 & 2176 & 1024 & 2112 & 1024 & 2112 \\
17 & 512 & 1056 & 1024 & 2816 & 448 & 1792 & 1024 & 2016 & 1024 & 2112 & 1024 & 2112 \\
18 & 512 & 1056 & 1024 & 2816 & 416 & 1664 & 1024 & 1856 & 1024 & 2112 & 1280 & 2656 \\
19 & 512 & 1056 & 1024 & 2816 & 512 & 2048 & 1024 & 1728 & 1024 & 2112 & 1280 & 2656 \\
20 & 512 & 1056 & 1024 & 2816 & 640 & 2560 & 1024 & 1600 & 1024 & 2112 & 1536 & 3168 \\
21 & 512 & 1056 & 1024 & 2816 & 768 & 3072 & 1024 & 1536 & 1024 & 2112 & 2048 & 4224 \\
22 & 512 & 1056 & 1024 & 2816 & 928 & 3712 & 1024 & 1472 & 1024 & 2112 & 2048 & 4224 \\
23 & 512 & 1056 & 1024 & 2816 & 1152 & 4608 & 1024 & 1408 & 1024 & 2112 & 2048 & 4224 \\
24 & 512 & 1056 & 1024 & 2816 & 1408 & 5632 & 1024 & 1408 & 1024 & 2112 & 1792 & 3712 \\
\bottomrule
\end{tabularx}
\end{table}

\section{LoNAS-LLaMA protocol and budget results}
\label{app:lonas}

For the supernet in Section~\ref{sec:lonas}, each block chooses an FFN width from $\{5504,6880,8256,9632,11008\}$ and a Q/K/V LoRA rank from $\{28,32\}$ \citep{hu2022lora}; O and FFN adapters retain rank 32. Both methods select rank 32 throughout. Table~\ref{tab:lonas-pareto} reports eight matched parameter budgets and the eight benchmark accuracies averaged into Avg8, including OpenBookQA \citep{mihaylov2018openbookqa}; Table~\ref{tab:lonas-configs} gives the configurations. Reverse reverses FFN-width assignments across the original pretrained layers. Labels A--G match Figure~\ref{fig:lonas-pareto}. The archived proxy searches use 50 individuals over 20 generations (1,000 evaluations) under a 5.7B ceiling \citep{zhu2026nsc}.

\paragraph{LoRA scoring.}
We add the contextual scores of the three Q/K/V adapter paths to the seven-projection score. Each path has $A\in\mathbb R^{r\times d}$ and $B\in\mathbb R^{d\times r}$, with entry variance $2/(d+r)$ and unit noise. Their source covariances are $\Sigma_{x,b}^{\mathrm{read}}$ for $A$ and $A^{\mathrm{ref}}\Sigma_{x,b}^{\mathrm{read}}(A^{\mathrm{ref}})^\top$ for $B$, where independent $A^{\mathrm{ref}}$ has entry variance $1/d$. These terms use the spectral evaluations in Appendix~\ref{app:spectral}, leaving the residual ledger unchanged. At fixed residual width they are independent of FFN width; rank adds no dense deployment cost after merging. Rank 32 scores higher at every layer on this menu.

\paragraph{Search cost and FLOPs.}
The 0.28\,s search time is the median of five fresh-process runs on one CPU thread, covering spectral scoring and joint FFN/rank selection at all eight budgets. NSC and proxy timings follow \citet{zhu2026nsc}. Figure~\ref{fig:lonas-pareto} counts dense forward FLOPs at sequence length 256, with two FLOPs per multiply-add, including projections, attention products, FFNs, and vocabulary readout. Matched ResNSC/NSC configurations and the highest-budget Reverse control have identical FLOPs.

\begin{table}[!htbp]
\centering
\caption{Complete LoNAS results with the shared FFN menu $\{5504,6880,8256,9632,11008\}$. ResNSC and NSC have identical parameter counts at each budget. Labels match Figure~\ref{fig:lonas-pareto}; ``best'' denotes the highest-budget point. $\Delta$ is the Avg$_8$ difference from the corresponding NSC model (percentage points).}
\label{tab:lonas-pareto}
\setlength{\tabcolsep}{3pt}
\resizebox{\linewidth}{!}{%
\begin{tabular}{@{}l r *{9}{r} r@{}}
\toprule
Method & Params & BoolQ & PIQA & SIQA & HSwag & WGrnd & ARC-e & ARC-c & OBQA & Avg$_8$ & $\Delta$ \\
\midrule
NSC best & 5.690\,B & 64.01 & 72.85 & 68.94 & 51.30 & 64.72 & 72.43 & 58.70 & 70.80 & 65.47 & --- \\
NSC A & 5.572\,B & 64.01 & 72.42 & 68.17 & 51.26 & 64.09 & 72.52 & 58.53 & 70.60 & 65.20 & --- \\
NSC B & 5.538\,B & 63.73 & 72.42 & 68.27 & 51.05 & 63.46 & 72.18 & 58.79 & 71.20 & 65.14 & --- \\
NSC C & 5.487\,B & 63.61 & 72.25 & 67.50 & 50.69 & 63.38 & 72.18 & 58.87 & 72.00 & 65.06 & --- \\
NSC D & 5.453\,B & 63.73 & 72.03 & 66.73 & 49.95 & 63.85 & 71.34 & 58.19 & 70.80 & 64.58 & --- \\
NSC E & 5.420\,B & 63.15 & 71.49 & 67.71 & 49.73 & 63.69 & 71.30 & 57.08 & 71.00 & 64.39 & --- \\
NSC F & 5.369\,B & 62.51 & 69.37 & 66.63 & 46.54 & 62.51 & 68.06 & 54.86 & 67.80 & 62.29 & --- \\
NSC G & 5.098\,B & 59.27 & 69.48 & 64.64 & 44.73 & 62.27 & 66.04 & 52.82 & 66.60 & 60.73 & --- \\
\midrule
ResNSC best & 5.690\,B & 61.71 & 76.71 & 70.98 & 58.39 & 67.01 & 78.45 & 61.69 & 76.00 & 68.87 & +3.40 \\
ResNSC A & 5.572\,B & 62.69 & 76.44 & 70.83 & 57.29 & 66.77 & 78.03 & 61.26 & 76.20 & 68.69 & +3.49 \\
ResNSC B & 5.538\,B & 63.73 & 76.55 & 70.83 & 57.08 & 66.30 & 77.99 & 61.35 & 75.80 & 68.70 & +3.56 \\
ResNSC C & 5.487\,B & 63.55 & 76.50 & 70.78 & 57.08 & 65.27 & 77.69 & 61.09 & 74.80 & 68.34 & +3.28 \\
ResNSC D & 5.453\,B & 63.55 & 76.12 & 70.37 & 57.36 & 65.19 & 77.82 & 60.15 & 74.20 & 68.09 & +3.51 \\
ResNSC E & 5.420\,B & 62.87 & 76.01 & 70.78 & 56.54 & 65.43 & 77.23 & 60.32 & 73.20 & 67.80 & +3.41 \\
ResNSC F & 5.369\,B & 64.37 & 76.01 & 70.37 & 55.82 & 65.04 & 76.98 & 59.22 & 74.20 & 67.75 & +5.46 \\
ResNSC G & 5.098\,B & 63.61 & 73.01 & 67.76 & 50.97 & 63.22 & 73.19 & 57.17 & 72.00 & 65.12 & +4.39 \\
\midrule
Reverse best & 5.690\,B & 61.25 & 67.46 & 63.10 & 42.64 & 61.17 & 65.24 & 51.45 & 63.60 & 59.49 & -5.98 \\
\bottomrule
\end{tabular}}
\end{table}

\begin{table}[!htbp]
\centering
\caption{Exact LoNAS FFN configurations for the budgets in Table~\ref{tab:lonas-pareto}. Each width column lists the layers using that FFN intermediate width; layers are numbered 1--32 from input to output, and a dash means none. All configurations have residual width 4,096; both search methods select LoRA rank 32 in every layer.}
\label{tab:lonas-configs}
\footnotesize
\setlength{\tabcolsep}{2.5pt}
\renewcommand{\arraystretch}{1.08}
\begin{tabular*}{\linewidth}{@{\extracolsep{\fill}}llrccccc@{}}
\toprule
 & & & \multicolumn{5}{c}{Layer indices by FFN width} \\
\cmidrule(lr){4-8}
Method & Budget & Params (B) & 11008 & 9632 & 8256 & 6880 & 5504 \\
\midrule
ResNSC & best & 5.690 & 1--9 & 10--13 & 14--18 & 19--26 & 27--32 \\
ResNSC & A & 5.572 & 1--8 & 9--12 & 13--16 & 17--23 & 24--32 \\
ResNSC & B & 5.538 & 1--8 & 9--11 & 12--16 & 17--22 & 23--32 \\
ResNSC & C & 5.487 & 1--7 & 8--11 & 12--15 & 16--21 & 22--32 \\
ResNSC & D & 5.453 & 1--7 & 8--10 & 11--14 & 15--21 & 22--32 \\
ResNSC & E & 5.420 & 1--6 & 7--10 & 11--14 & 15--20 & 21--32 \\
ResNSC & F & 5.369 & 1--6 & 7--9 & 10--13 & 14--19 & 20--32 \\
ResNSC & G & 5.098 & 1--3 & 4--6 & 7--9 & 10--13 & 14--32 \\
\midrule
NSC & best & 5.690 & --- & 1--2 & 3--32 & --- & --- \\
NSC & A & 5.572 & --- & --- & 1--15, 21--32 & 16--20 & --- \\
NSC & B & 5.538 & --- & --- & 1--13, 21--32 & 14--20 & --- \\
NSC & C & 5.487 & --- & --- & 1--11, 21--31 & 12--20, 32 & --- \\
NSC & D & 5.453 & --- & --- & 1--9, 21--31 & 10--20, 32 & --- \\
NSC & E & 5.420 & --- & --- & 1, 3--8, 21--31 & 2, 9--20, 32 & --- \\
NSC & F & 5.369 & --- & --- & 6--9, 21--31 & 1--5, 10--20, 32 & --- \\
NSC & G & 5.098 & --- & --- & --- & 1--31 & 32 \\
\midrule
Reverse & best & 5.690 & 24--32 & 20--23 & 15--19 & 7--14 & 1--6 \\
\bottomrule
\end{tabular*}
\end{table}

\section{Architecture-ranking benchmarks and protocol}
\label{app:ranking}
\paragraph{Evaluation sets and targets.}
We reuse the existing architecture IDs and performance labels in four benchmark sets: 500 FlexiBERT entries with widths read from each architecture configuration, 200 GPT-2 entries, and 1,001 entries each in AutoFormer-Tiny and AutoFormer-Small. Correlations use GLUE for FlexiBERT, negative validation perplexity for GPT-2 (test perplexity where the validation field is absent), and clean ImageNet accuracy for AutoFormer. All scores are computed analytically without candidate training or forward passes.
\paragraph{Architecture and variance conventions.}
FlexiBERT reads widths from each architecture configuration and retains the archived inventory: merged QKV for self-attention, two square matrices for the LT operator, one square matrix for DSC, and the archived paired-matrix decomposition for multiple FFNs. This multiple-FFN inventory follows the benchmark scoring representation; the source network implements the FFNs sequentially. GPT-2 uses the archived entry standard deviation 0.02. The remaining inventories retain the NSC shape-based reference variances. AutoFormer uses head dimension 64 and internal attention width $64H$, including a rectangular output projection. Both scores use the same architectural specifications; their Q/K/V evaluation conventions are distinguished below.
\paragraph{Residual and FFN adaptation.}
Each candidate has constant residual width across its layers, giving normalized branch-input covariance $\Sigma_{x,b}^{\mathrm{in}}=I$ and scalar depth-dependent read/write metrics. The same residual reference is applied across these archived sets, including their post-normalized models. Ordinary FFNs have one up/read projection and one propagated down/write projection, with Gaussian hidden reference and unit gate. The propagation model uses this Gaussian closure in place of a GELU covariance kernel. For GPT-2 and AutoFormer, NSC retains its headwise decomposition, while ResNSC evaluates complete Q/K/V projections with the inherited per-head entry variance and uses a normalized propagated value-to-output channel.
\begin{table}[!htbp]
\centering
\caption{Detailed ranking correlations. Parameter-count correlations retain the archived counting convention.}
\label{tab:rsc-ranking-full}
\scriptsize
\setlength{\tabcolsep}{3pt}
\begin{tabularx}{\linewidth}{@{}Xrrrrrrr@{}}
\toprule
 & \multicolumn{2}{c}{NSC} & \multicolumn{2}{c}{ResNSC} & \multicolumn{2}{c}{Params} & \\
\cmidrule(lr){2-3}\cmidrule(lr){4-5}\cmidrule(lr){6-7}
Benchmark & $\tau_b$ & $\rho$ & $\tau_b$ & $\rho$ & $\tau_b$ & $\rho$ & $\Delta\rho$ \\
\midrule
FlexiBERT & 0.6946 & 0.8843 & 0.6921 & 0.8833 & 0.4846 & 0.6763 & -0.0009 \\
GPT-2 & 0.8488 & 0.9681 & 0.8543 & 0.9698 & 0.6386 & 0.7943 & +0.0018 \\
AutoFormer-Tiny & 0.6119 & 0.8106 & 0.5717 & 0.7786 & 0.4691 & 0.6347 & -0.0320 \\
AutoFormer-Small & 0.6580 & 0.8519 & 0.7192 & 0.9070 & 0.7885 & 0.9414 & +0.0551 \\
\bottomrule
\end{tabularx}
\end{table}
\paragraph{Ranking within a 10\% parameter window.}
We retain unordered architecture pairs satisfying $|P_i-P_j|/\max(P_i,P_j)<0.10$, using the same archived parameter counts, performance labels, and frozen NSC/ResNSC scores as above. For each score, let $C$ and $D$ count concordant and discordant pairs after excluding ties in either the score or the performance label. Pairwise accuracy is $C/(C+D)$ and $\tau_{\mathrm{PW}}=(C-D)/(C+D)$. The pair count in Table~\ref{tab:rsc-ranking-pw10} is measured before tie exclusion; $\tau_{\mathrm{PW}}$ uses this restricted pair set rather than the global $\tau_b$ statistic.

\begin{table}[!htbp]
\centering
\caption{Parameter-controlled pairwise ranking ($\uparrow$). Only pairs with less than 10\% relative parameter difference are compared. Accuracy is reported in percent; $\Delta$ is ResNSC minus NSC accuracy in percentage points. Bold marks the higher value within each metric and row.}
\label{tab:rsc-ranking-pw10}
\small
\setlength{\tabcolsep}{4pt}
\renewcommand{\arraystretch}{1.15}
\begin{tabular*}{\linewidth}{@{\extracolsep{\fill}}l r rr rr r@{}}
\toprule
 & & \multicolumn{2}{c}{PW accuracy (\%)} & \multicolumn{2}{c}{$\tau_{\mathrm{PW}}$} & \\
\cmidrule(lr){3-4}\cmidrule(lr){5-6}
Benchmark & Pairs & NSC & ResNSC & NSC & ResNSC & $\Delta$ (pp) \\
\midrule
FlexiBERT & 26,769 & \textbf{75.15} & 74.98 & \textbf{0.5031} & 0.4997 & -0.17 \\
GPT-2 & 1,721 & 77.28 & \textbf{78.27} & 0.5456 & \textbf{0.5654} & +0.99 \\
AutoFormer-Tiny & 372,469 & \textbf{75.32} & 72.51 & \textbf{0.5063} & 0.4503 & -2.80 \\
AutoFormer-Small & 130,265 & \textbf{67.42} & 66.16 & \textbf{0.3483} & 0.3232 & -1.25 \\
\bottomrule
\end{tabular*}
\end{table}
\noindent Using only pairs for which the performance label and both scores are non-tied gives the same direction of change on all four sets. Under this common-pair protocol, the ResNSC accuracy differences are $-0.09$, $+0.99$, $-2.72$, and $-1.25$ percentage points in table order.
\FloatBarrier

\end{document}